\documentclass[sigconf]{acmart}

\usepackage{algpseudocode,algorithm,algorithmicx}
\usepackage{lipsum}
\usepackage{cases}
\usepackage{subfiles}
\usepackage{diagbox}
\usepackage{multirow}

\theoremstyle{definition}
\newtheorem{defn}{Definition}[section]
\newtheorem{prop}{Proposition}[section]

\theoremstyle{remark}

\usepackage{makecell}

\usepackage{graphicx}
\usepackage{subfig}
\usepackage{float}

\AtBeginDocument{%
  \providecommand\BibTeX{{%
    \normalfont B\kern-0.5em{\scshape i\kern-0.25em b}\kern-0.8em\TeX}}}

\acmISBN{978-1-4503-XXXX-X/18/06}

\begin{document}

\title{Learn to Tour: Operator Design For Solution Feasibility Mapping in Pickup-and-delivery Traveling Salesman Problem}


\author{Bowen Fang}
\affiliation{%
  \institution{Columbia University}
  \city{New York}
   \country{United States}
}
\email{bf2504@columbia.edu}

 \author{Xu Chen}
 \affiliation{
   \institution{Columbia University}
   \city{New York}
   \country{United States}}
 \email{xc2412@columbia.edu}

  \author{Xuan Di}
 \affiliation{
   \institution{Columbia University}
   \city{New York}
   \country{United States}}
 \email{sharon.di@columbia.edu}

\begin{abstract}
This paper aims to develop a learning method for a special class of traveling salesman problems (TSP), namely, the pickup-and-delivery TSP (PDTSP), which finds the shortest tour along a sequence of one-to-one pickup-and-delivery nodes. One-to-one here means that the transported people or goods are associated with designated pairs of pickup and delivery nodes, in contrast to that indistinguishable goods can be delivered to any nodes. In PDTSP, precedence constraints need to be satisfied that each pickup node must be visited before its corresponding delivery node. Classic operations research (OR) algorithms for PDTSP are difficult to scale to large-sized problems. Recently, reinforcement learning (RL) has been applied to TSPs. The basic idea is to explore and evaluate visiting sequences in a solution space. However, this approach could be less computationally efficient, as it has to potentially evaluate many infeasible solutions of which precedence constraints are violated. To restrict solution search within a feasible space, we utilize operators that always map one feasible solution to another, without spending time exploring the infeasible solution space. Such operators are evaluated and selected as policies to solve PDTSPs in an RL framework. We make a comparison of our method and baselines, including classic OR algorithms and existing learning methods. Results show that our approach can find tours shorter than baselines. 
\end{abstract}




\maketitle

\section{Introduction}
\label{sec:intro}
The pickup-and-delivery traveling salesman problem (PDTSP) finds a minimum cost tour (i.e., a sequence of pickup and their corresponding delivery nodes) on a graph \cite{Savelsbergh1990pd}. There exists a one-to-one correspondence between a pickup and a delivery node, where precedence constraints need to be fulfilled that each pickup node must be visited before its corresponding delivery node. In other words, passengers cannot be dropped off before getting picked up.

PDTSP has important implications for real-world applications such as flexible shuttle and vendor delivery services. In a flexible shuttle system shown in Figure~\ref{subfig:ride}, two passengers request a shared ride service during the same time window. The shuttle, departing and returning the same depot, needs to find an optimal sequence to transport these two customers from their origins to respective destinations using a shortest travel time or minimum distance traveled. Here $depot\rightarrow o_1 \rightarrow o_2 \rightarrow d_2 \rightarrow d_1 \rightarrow depot$ forms a feasible pick-up and delivery sequence. Similarly, in one-to-one delivery service shown in Figure~\ref{subfig:delivery}, a truck, starting and ending from the same depot, is tasked to deliver goods between two source and two target warehouses, with a goal of minimizing total fuel consumption (which could be a function of total distance traveled). Here $depot \rightarrow s_1 \rightarrow t_1 \rightarrow s_2 \rightarrow t_2 \rightarrow depot$ forms a feasible sequence.

\begin{figure}[h]
    \centering
    \vspace{-0.5cm}
    \subfloat[Flexible shuttle service where $o$ is origin and $d$ is destination.]{\includegraphics[scale=.4]{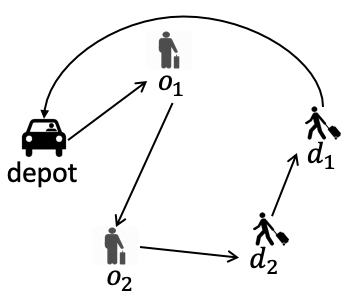}\label{subfig:ride}}
    \hspace{0.6cm}
    \subfloat[Vendor delivery service where $s$ is source and $t$ is target warehouse.]{\includegraphics[scale=.4]{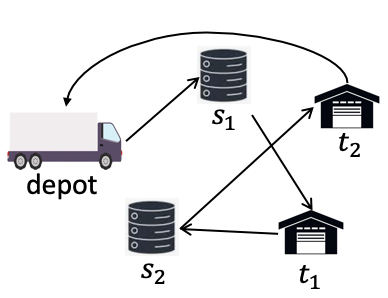}\label{subfig:delivery}}
    \vspace{-0.2cm}
    \caption{Motivating examples}
    \label{fig:moti}
    \vspace{-0.5cm}
\end{figure}

Finding optimal or near-optimal tours in PDTSP is challenging \cite{pacheco2022pdtsp}, because there could exist a large number of infeasible solutions (i.e., visiting sequences that violate precedence constraints). To elaborate, for a PDTSP with $n$ pickup-delivery node pairs (i.e., $\mathcal{N}=2n$ nodes), the total number of visiting sequences is $(2n)!$, while feasible ones only account for $\frac{(2n)!}{2^n}$ \cite{ruland1994pd}. As the problem size increases, $\lim_{n \rightarrow \infty} \frac{(2n)!/2^n}{(2n)!}=0$, indicating that feasible solutions account for merely a small portion of all the visiting sequences. Thus, searching for feasible solutions within the feasible space is challenging in PDTSP.

\begin{figure}[h]
   \centering
   \vspace{-0.1in}
   \includegraphics[scale=.115]{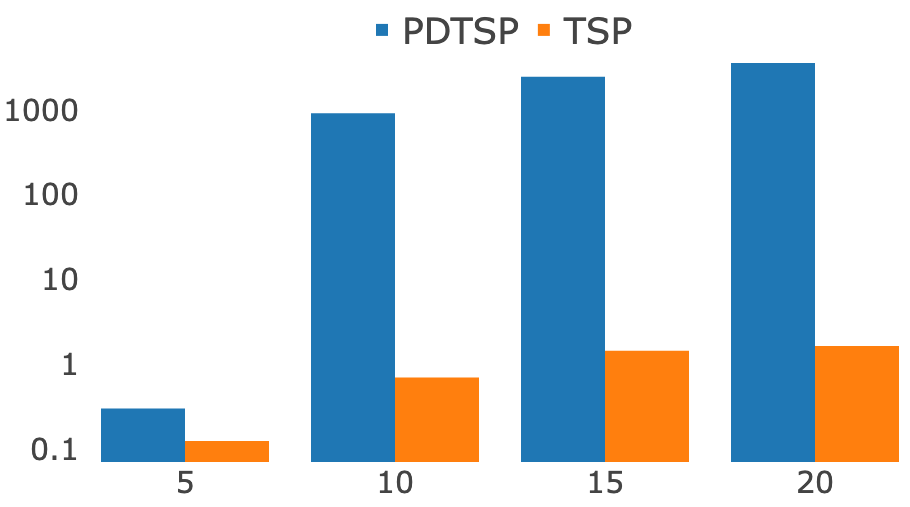}
   \includegraphics[scale=.338]{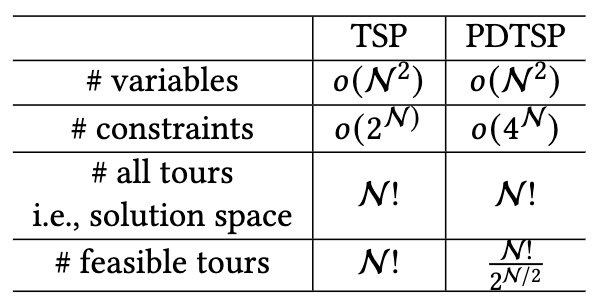}
   \vspace{-0.15in}
    \caption{PDTSP vs TSP. In the left bar chat, we list the computational time for PDTSP and TSP. The x-axis denotes the number of node $|\mathcal{N}|$ and the y-axis is the log of the compulational time (s). As the size of instance grows, the solving time for PDTSP is thousand times longer than TSP. We list the number of variables, constraints and feasible tours given the same number of nodes $|\mathcal{N}|$ to be visited.}
    \label{fig:cmp_fig1}
    \vspace{-0.2in}
\end{figure}

To tackle such a challenge, this paper leverages learning operators that consistently map one feasible solution to another. The learning operators enable us to efficiently generate feasible tours without exploring the entire space dominated by infeasible solutions. Our contributions are as follows:
\begin{enumerate}
    \item We propose a unified operator set containing various operators that map one feasible tour to another by confining the solution search within a feasible space. We prove that operators in the unified operator set can always generate feasible Hamiltonian cycles in PDTSP.
    \item We utilize the unified operator set as the action set in an RL framework for solving PDTSP, where operators are evaluated and selected as policies given the current state (i.e., a feasible solution of PDTSP).
    \item We demonstrate the effectiveness of our method by conducting experiments on various problem scales. We provide a comparative analysis of computational efficiency and solution quality between our method and off-the-shelf solvers.
\end{enumerate}

\begin{table}[ht]
\vspace{-0.2in}
\begin{center}
	\begin{tabular}{c|c|c|c|c}
	\hline
         & Setup & \makecell{Operator \\ Design} & \makecell{Problem \\ scale} & \makecell{Training \\time}  \\
        \hline
         Ours & PDTSP & \makecell{Unified\\ operators}  & $\mathcal{N}=501$ & 12 hours \\
         \hline
         \cite{ma2022pd} & PDTSP & \makecell{Single \\operator} & $\mathcal{N}=101$ & 7 days\\ 
        \hline
         \cite{lu2020A} & CVRP & \makecell{Intra-\& inter \\vehicle \\operators}& $\mathcal{N}=101$ &--\\
        \hline 
	\end{tabular}
\end{center}
\caption{Existing frameworks with learning operator}\label{table:compare_tsp}
\end{table}
\vspace{-0.3in}
The rest of this paper is organized as: Section \ref{sec:related} discusses related work on PDTSP. Section \ref{sec:pre} introduces preliminaries. In Section \ref{sec:method}, we present our methodology, where learning operators are designed within an RL framework. The solution approach is detailed in Section \ref{sec:sol}, followed by numerical results in Section \ref{sec:numer}. Section \ref{sec:conclu} concludes.

\section{Related Work}
\label{sec:related}
Classic OR algorithms for PDTSP include exact and heuristic methods. The widely used exact method is the branch-and cut algorithm \cite{dumi2009tsp}, which solves the exact solution of PDTSP via a binary integer programming \cite{ruland1994pd}. However, the exact method is difficult to scale to large-sized problems, because the number of precedence constraints significantly increases as the number of pickup-and-delivery nodes increases. Heuristic methods, adapted from neighbourhood search, are also proposed to solve PDTSP within a reasonable time \cite{renaud2000pd,renaud2002pd,veenstra2017pdtsp}. The well-known Lin-Kernighan-Helsgaun (LKH) solver \cite{lin1973AnEH} is extended to LKH3.0 \cite{lkh3,pacheco2022pdtsp} to solve TSP invariants including PDTSP.   

Recent years have seen a growing trend of using learning methods to solve classic TSPs. Most of them use an end-to-end approach in which deep neural networks that output TSP solutions are trained using either supervised learning \cite{vinyals2015ptr,Joshi2019graph,xin2021neurolkh} or reinforcement learning (RL) \cite{bello2016rl,dai2017rl,nazari2018rl,kool2018attention,miki2018rl,barrett2020rl,Xing2020TS,Duan2020EfficientlyST}. These learning frameworks requires a large number of problem instances during the training process and cannot outperform classic OR algorithms (e.g., LKH3) in test instances \cite{lu2020A}. Many other studies utilize RL to enhance neighbourhood search for classic TSPs and VRPs \cite{deudon2018tsp,chen2019encode,wu2019rl,Costa2020rl,Zheng2020CombiningRL,ma2021a,zong2022multi,Zheng2023lkh}. However, these approaches do not consider precedence constraints, and thus are not applicable to PDTSP. 

To accommadate precedence constraints, \cite{lixijun2021pd,lijing2022pd} utilize a mask mechanism in which delivery nodes are masked in policy networks before the corresponding pickup nodes are visited. 
To make policy networks identify a tour, a large number of solutions, including both feasible and infeasible ones, have to be sampled from the entire solution space. The training process could be time-consuming, especially when infeasible solutions dominate the space. To tackle this issue, researchers seek operators that confine searching within a feasible solution space and reduce the search space in PDTSP. To search within the feasible space, \cite{carrabs2007operator} utilizes sub-sequences containing only corresponding pickup-and-delivery nodes and swap node pairs to obtain a new tour. \cite{ma2022pd} adopts an insertion operator that place nodes on positions fulfilling precedence constraints, in order to explore feasible solutions.

Here we extend the operators in \cite{carrabs2007operator,ma2022pd} to a unified operator set containing  various admissible operators that transform one feasible tour to another. Instead of using one operator (e.g., insertion) to search solutions, we use the unified operator set in an RL framework to explore the solution space.








\section{Problem Statement}
\label{sec:pre}
In this section, we briefly introduce PDTSP and define terminologies related to solutions of PDTSP. 

\subsection{A primer on pickup and delivery traveling salesman problem (PDTSP)}
PDTSP is defined on a graph with pickup and delivery nodes. 
On a graph $\mathcal{G}=\{\mathcal{N},\mathcal{L}\}$ where $\mathcal{N}$ is the node set and $\mathcal{L}$ is the edge set, a vehicle, starting from a depot node $0$, has to fulfill $n$ requests corresponding to each pickup-and-delivery nodes pairs $(i,n+i)$ ($1\leqslant i \leqslant n$). The pickup and delivery node sets are denoted by $\mathcal{P}=\{1,\cdots,n\}$ and $\mathcal{D}=\{n+1,\cdots,2n\}$, respectively. 
Thus, the node set $\mathcal{N}=\ \mathcal{P} \cup \mathcal{D} \cup \{0\}$, and its cardinality is $|\mathcal{N}|=2n+1$.
The vehicle returns to the same depot node after fulfilling all requests. The travel cost between two nodes $i,j\in\mathcal{N}$ is denoted by $c_{ij}$. 
The goal is to minimize the total cost for the vehicle. 
The integer programming (IP) for PDTSP \cite{dumi2009tsp} is provided in Appendix \ref{append:ip}. 

From now on, we will use shorthand notations that refers a pickup node as $\mathcal{P}$-node and a delivery node as $\mathcal{D}$-node, respectively.


\begin{defn}\label{defn:visitseq} \textbf{Visiting sequence:} A visiting sequence is a sequence of nodes visited on a graph.  
\end{defn}
\begin{defn}\label{defn:hamiltonian} \textbf{Hamiltonian cycle:} 
A Hamiltonian cycle is a closed graph cycle through a graph that visits each node exactly once \cite{skiena1991hc}, denoted as $\ 0\rightarrow i \rightarrow \cdots \rightarrow n+j \rightarrow 0$ ($1\leqslant i,j \leqslant n$) where each $\mathcal{P}$- and $\mathcal{D}$-node is visited exactly once. 
\end{defn}
\begin{defn}
\label{defn:tour} \textbf{Tour in PDTSP:} A tour $\mathcal{T}$ is a feasible Hamiltonian cycle. 
A Hamiltonian cycle is feasible if the following \textbf{precedence conditions} (PC) hold $\forall i \in \{1,\cdots,n \}$:
\begin{align}
    (\text{\small{Precedence constraints for nodes})  } p_i<d_{n+i}, 
\end{align}
where $p_i, d_{n+i} (1 \leqslant p_i,d_{n+i} \leqslant 2n)$ are the indices of $\mathcal{P}$-node $i$ and $\mathcal{D}$-node $n+i$ in a tour, respectively. 
\end{defn}

\subsection{Why is feasible solution mapping important?}
We introduce a property regarding tours in PDTSP to highlight the importance of feasible solution mapping. 
\begin{prop}\label{prop:num_tour}
The total number of Hamiltonian cycle in PDTSP is $(2n)!$. The total number of feasible Hamiltonian cycles (i.e., tours) is $\frac{(2n)!}{2^n}$.
\end{prop}


The proof is in Appendix \ref{appendix:proof}. Proposition \ref{prop:num_tour} indicates that the feasible solution space is considerably smaller than the whole solution space, as the problem size increases. This motivates us to develop more efficient tools to restrict solution search in a feasible space without exploring the infeasible one. 


\subsection{Blocks in a tour}
Here we define pickup and delivery blocks, which satisfy precedence constraints and facilitate operator design.   

\begin{defn}\label{defn:block} \textbf{Pickup and delivery blocks ($\mathcal{P}$- and $\mathcal{D}$-blocks):} We define a pickup (delivery) block as a sub-sequence in a tour, containing only pickup (delivery) nodes. We use $\mathcal{T}^{(\mathcal{P})}$ and $\mathcal{T}^{(\mathcal{D})}$ to denote $\mathcal{P}$- and $\mathcal{D}$-blocks, respectively. Mathematically,
\begin{align}
    & \mathcal{P}\text{-block } \mathcal{T}^{(\mathcal{P})}: i_1 \rightarrow \cdots \rightarrow i_l \nonumber  \\
    & \mathcal{D}\text{-block } \mathcal{T}^{(\mathcal{D})}: n+j_1 \rightarrow\cdots \rightarrow n+j_m \nonumber 
\end{align} 
where, $l \geqslant 1$ is the size of $\mathcal{T}^{(\mathcal{P})}$ and $m \geqslant 1$ is the size of $\mathcal{T}^{(\mathcal{D})}$. We use $\mathcal{N}(\mathcal{T}^{(\mathcal{P})})$ and $\mathcal{N}(\mathcal{T}^{(\mathcal{D})})$ to denote the sets of nodes belong to $\mathcal{P}$-block $\mathcal{T}^{(\mathcal{P})}$ and $\mathcal{D}$-block $\mathcal{T}^{(\mathcal{D})}$, respectively. $\mathcal{N}(\mathcal{T}^{(\mathcal{P})}) \subseteq \mathcal{P}$ and $\mathcal{N}(\mathcal{T}^{(\mathcal{D})}) \subseteq \mathcal{D}$.
\end{defn}

Below we present properties of $\mathcal{P}$- and $\mathcal{D}$-blocks for the representation of blocks in a tour.

\begin{prop}\label{prop:tour_block}
    A tour can always be represented by a sequence of $\mathcal{P}$- and $\mathcal{D}$-blocks.
\end{prop}

\begin{prop}\label{prop:block_sub}
    A block with size$\geqslant 2$ can be decomposed as two adjacent blocks. On the contrary, any two adjacent $\mathcal{P}$- ($\mathcal{D}$-) blocks can be combined as one $\mathcal{P}$- ($\mathcal{D}$-) block.
\end{prop}
Proposition \ref{prop:block_sub} indicates that given a tour in PDTSP, the block representation is not unique. 




\begin{prop}\label{prop:tour_precedence} \textbf{(Precedence constraints for blocks.)}
Denote $k_1$ and $k_2$ as indices of $\mathcal{P}$-block $\mathcal{T}^{(\mathcal{P})}_{k_1}$ and $\mathcal{D}$-block $\mathcal{T}^{(\mathcal{D})}_{k_2}$, 
and $\mathcal{N}(\mathcal{T}^{(\mathcal{P})}_{k_1}),\mathcal{N}(\mathcal{T}^{(\mathcal{D})}_{k_2})$ as the sets of nodes belonging to $\mathcal{P}$-block $\mathcal{T}_{k_1}^{(\mathcal{P})}$ and $\mathcal{D}$-block $\mathcal{T}_{k_2}^{(\mathcal{D})}$, respectively.
For a PD node pair $(i,n+i)$ in a tour represented by a block sequence, if $i \in \mathcal{N}(\mathcal{T}^{(\mathcal{P})}_{k_1})$ and $n+i \in \mathcal{N}(\mathcal{T}^{(\mathcal{D})}_{k_2})$, then $k_1<k_2$. 
\end{prop}

The proof of Proposition 3.2-3.4 is in Appendix \ref{appendix:proof}. This proposition says that if two nodes of a PD node pair fulfills precedence conditions, the $\mathcal{P}$- and $\mathcal{D}$- blocks that the PD node pair belongs to also fulfill precedence conditions for blocks. This finding would facilitate the design of our operators to be detailed in Sec.~\ref{sec:opt_design}.

Figure \ref{fig:5_node_tour} illustrates a tour $\mathcal{T}$ on a toy example when $n=5$. The blue dot $0$ is the depot. The green and red dots represent $\mathcal{P}$- and $\mathcal{D}$-nodes, respectively. The tour is $0 \rightarrow 1 \rightarrow 2 \rightarrow 3 \rightarrow 7 \rightarrow 8 \rightarrow 4 \rightarrow 5 \rightarrow 6 \rightarrow 9 \rightarrow 10 \rightarrow 0$ (marked by black arrow). Green and red circles mark $\mathcal{P}$- and $\mathcal{D}$-blocks, respectively. A block representation of the tour is $ 0 \rightarrow \mathcal{T}^{(\mathcal{P})}_1 \rightarrow \mathcal{T}^{(\mathcal{D})}_2 \rightarrow \mathcal{T}^{(\mathcal{P})}_3 \rightarrow \mathcal{T}^{(\mathcal{D})}_4 \rightarrow \mathcal{T}^{(\mathcal{D})}_5 \rightarrow 0$. We will use this toy example throughout Section \ref{sec:method} to demonstrate how to perform operators in our methodology.
\begin{figure}[h]
   \centering
   \includegraphics[scale=.35]{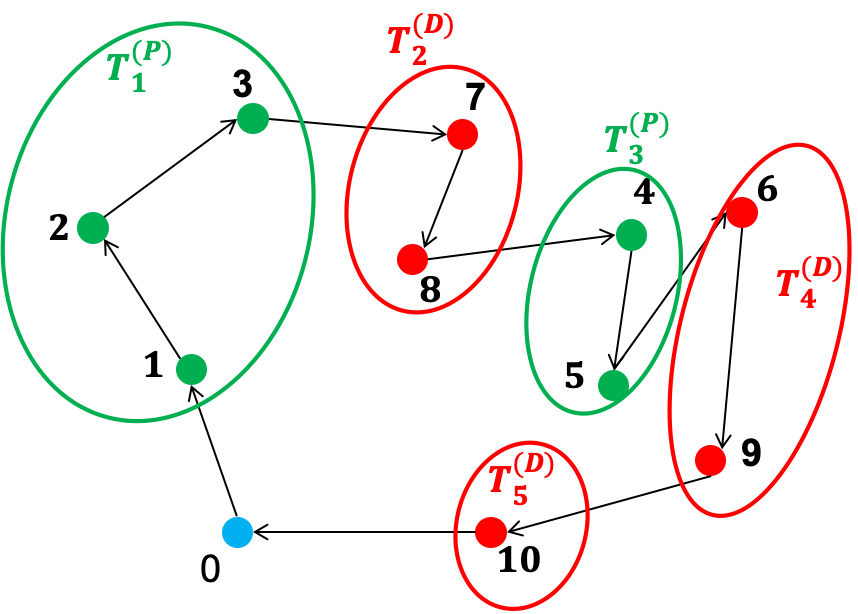}
   \vspace{-0.15in}
    \caption{Tour in a toy example ($n=5$)}
    \label{fig:5_node_tour}
\end{figure}


\section{Methodology}
\label{sec:method}
\subsection{RL framework}
\label{sec:rl_framework}

\begin{figure*}[t]
   \centering
   \includegraphics[scale=.45]{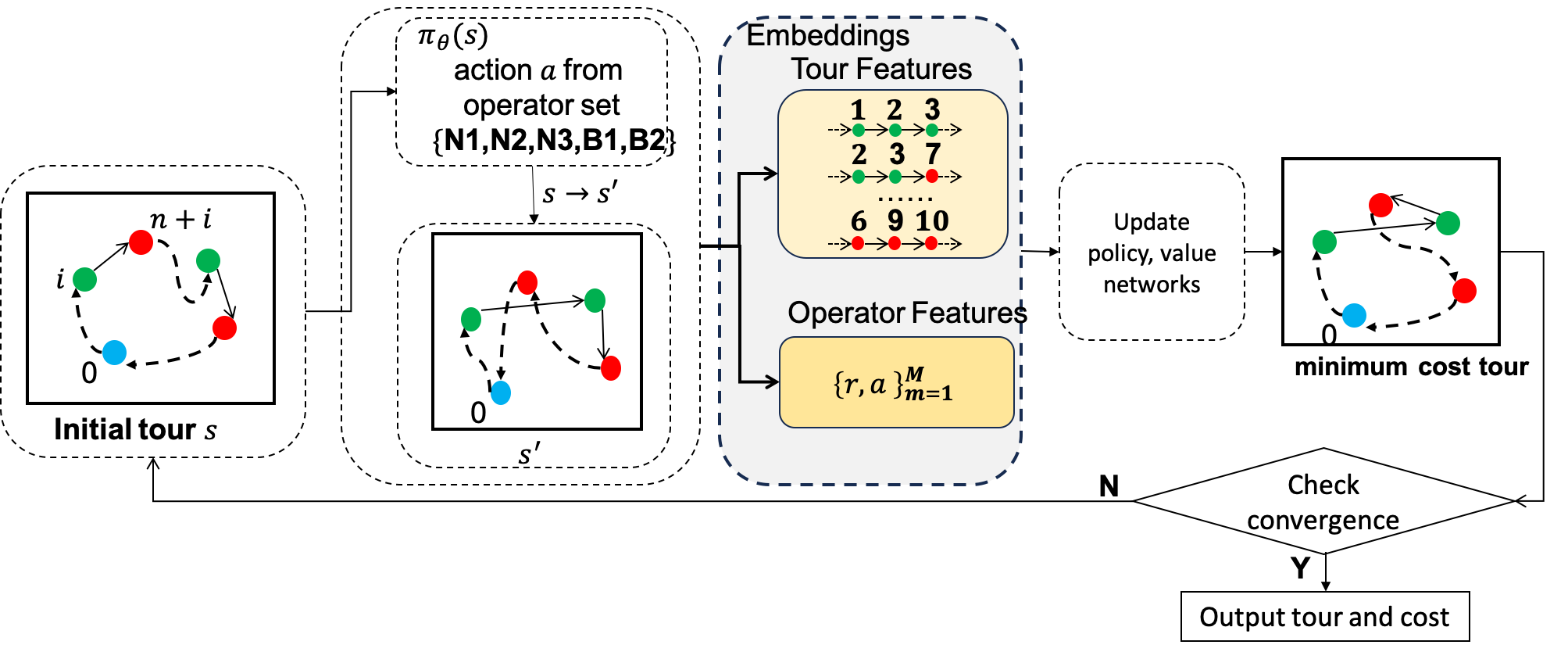}
   \vspace{-0.2in}
    \caption{RL framework}
    \label{fig:rl_frame}
    \vspace{-0.15in}
\end{figure*}

Let us establish an equivalence between RL and solving PDTSP. To solve PDTSP requires us (i.e., an agent) to identify a sequence of operations (i.e., actions) from an initial state (i.e., initial tour), in order to minimize total length of a tour (Figure \ref{fig:rl_frame}). We specify each element in the RL framework as:

\begin{itemize}
    \item \textbf{State $s\in \mathcal{S}$:} The current tour $\mathcal{T}$, including node and block information in the tour. The state space $\mathcal{S}$ encloses all $\frac{(2n)!}{2^n}$ tours (i.e., feasible Hamiltonian cycles).
    \item \textbf{Action $a\in{\mathcal{A}}$:} Operators that are designed to map one tour to another, including two types of operation, namely, node-exchange operators and block-exchange operators. Among them, there are a total of five discrete operators, namely, intra-block node-exchange operator ($N_1$), inter-block node-exchange operator ($N_2$), node pair-exchange operator ($N_3$), same type block-exchange operator ($B_1$), and mixed type block-exchange operator ($B_2$). Each operator is detailed in Section \ref{sec:opt_design}.
    \item \textbf{Transition $s\rightarrow s'$:} After an operator is executed, the current tour is updated to a new tour under the operation.
    \item \textbf{Reward $r$:} The reward of an operator is the cost difference between the current tour and the new one. Each action is associated with a reward, namely, $r_{N_1}, r_{N_2}, r_{N_3}, r_{B_1}, r_{B_2}$, respectively.
\end{itemize}

\noindent The key to implementing the RL method is to construct a feasible solution, and then define admissible actions that map one tour (i.e., feasible solution) to another. 

\subsection{Initial tour construction} 
\label{sec:generate_initial_tour}

We demonstrate how to construct initial tours, which can be used as initial states in the RL framework.
The following proposition shows that we can always construct a set of visiting sequences that are feasible.

\begin{prop}\label{prop:exist_tour}
If a visiting sequence satisfies: $p_i=d_{n+i}-1, \forall i \in \mathcal{P}$  where $p_i, d_{n+i}$ are the indices of $\mathcal{P}$-node $i$ and $\mathcal{D}$-node $n+i$, respectively, 
this sequence is a tour in PDTSP. 
In other words, the constructed tour where $p_i=d_{n+i}-1$ is when the vehicle responds to one and only one request at a time until all passengers or goods are delivered. 
\end{prop}
\begin{proof}
   Consider a sequence $0\rightarrow i \rightarrow n+i \cdots\rightarrow j \rightarrow n+j \rightarrow 0, \forall i \in \mathcal{P}$ where $p_i=d_{n+i}-1$. Therefore, the sequence is a tour in PDTSP according to Definition \ref{defn:tour}. 
\end{proof}

Proposition \ref{prop:exist_tour} facilitates the generation of initial tours in our learning framework by randomizing the order of $n$ requests. 
Here we briefly describe how to construct a tour based on Proposition \ref{prop:exist_tour}. 
Given an instance with $n$ PD node pairs, we can randomly generate a permutation of the $\mathcal{P}$-node set, denoted as $(i_1,\cdots,i_n)$. We then construct a tour as $0\rightarrow i_1 \rightarrow n+i_1 \rightarrow \cdots \rightarrow i_n\rightarrow n+i_n \rightarrow 0$, by appending the $\mathcal{D}$-nodes after their corresponding $\mathcal{P}$-nodes. 

\subsection{Learning operator design} 
\label{sec:opt_design}
This section shows how to design admissible operators for PDTSP. We first briefly introduce a naive operator which does not take into account the solution feasibility.

\noindent \textbf{Naive learning operator}
A naive operator randomly swaps two nodes in a tour. In Figure \ref{fig:naive_operator}, we demonstrate the operator using a sequence of three nodes. The operator swaps  node $j$ and $k$. The visiting sequence in the solution is changed from $i \rightarrow j \rightarrow k$ to $i \rightarrow k \rightarrow j$. Note that the naive operator may violate precedence constraints. 

\begin{figure}[h]
    \centering
    \vspace{-0.1in}
    \includegraphics[scale=.3]{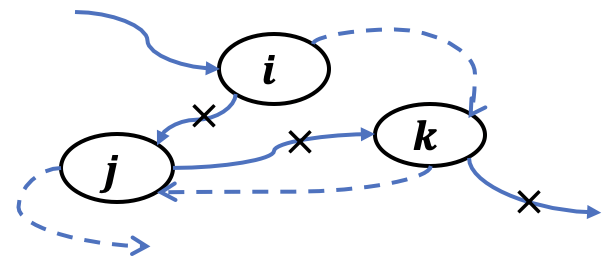}
    \vspace{-0.2in}
    \caption{Naive operator}
    \label{fig:naive_operator}
    \vspace{-0.2in}
\end{figure}


Below we propose five operators (belonging to two categories) that map one feasible solution to another. 
The conditions under which each operator can be adopted is also discussed. 

\begin{enumerate}
\item Node-exchange operator (N$_X$O)
\begin{enumerate}
\item[($N_1$)] Intra-block node-exchange operator (IntraBlock N$_X$O): 
It swaps the visiting sequence of two nodes within a $\mathcal{P}$- or $\mathcal{D}$-block (Figure \ref{subfig:o_1}). 

\textbf{Conditions:} 
Since there is no precedence constraint for nodes within a block, the visiting sequence of pickup (delivery) nodes in a $\mathcal{P}$- or $\mathcal{D}$-block can be randomly swapped.

\item[($N_2$)] Inter-block node-exchange operator (InterBlock N$_X$O):
It swaps the visiting sequence of two nodes across a $\mathcal{P}$-block and a $\mathcal{D}$-block (Figure \ref{subfig:o_2}). 

\textbf{Conditions:} The visiting sequence of nodes in a $\mathcal{D}$-block and nodes in $\mathcal{P}$-blocks following the $\mathcal{D}$-block can be randomly swapped. Because there is no precedence constraint for a $\mathcal{P}$- ($\mathcal{D}$-) node and any ones before (after) it.

\item[($N_3$)] Node pair-exchange operator (N$^2_X$O): It swaps $\mathcal{P}$-nodes from two PD node pairs, and accordingly also $\mathcal{D}$-nodes in the node pairs (Figure \ref{subfig:o_3}). 

\textbf{Conditions:} The $\mathcal{P}$- and $\mathcal{D}$-nodes of a node pair can be swapped with $\mathcal{P}$- and $\mathcal{D}$-nodes of another node pair. 
\end{enumerate}

\item Block-exchange operator (B$_X$O) 
Analogous to node-exchange operator, block-change operator swaps the visiting sequence of two blocks, of the same or different types. 
\begin{enumerate}
\item[($B_1$)] Same type block-exchange operator (SameB$_X$O):
It swaps the visiting sequence of two same-type blocks within a $\mathcal{P}$ ($\mathcal{D}$)-block sequence in a tour (Figure \ref{subfig:o_4}).

\textbf{Conditions:} Since adjacent same-type blocks can be combined as one block (Proposition \ref{prop:block_sub}) and there is no precedence constraints for nodes within a block, the visiting sequence of $\mathcal{P}$ ($\mathcal{D}$)-blocks in a block sequence can be randomly swapped.
\item[($B_2$)] Mixed type block-exchange operator (MixB$_X$O): MixB$_X$O swaps the visiting sequence of a $\mathcal{D}$-block and a $\mathcal{P}$-block in a tour (Figure \ref{subfig:o_5}). 

\textbf{Conditions:} The sequence of a $\mathcal{D}$-block and any $\mathcal{P}$-blocks following this $\mathcal{D}$-block can be randomly swapped. Because there is no precedence constraint for a $\mathcal{P}$- ($\mathcal{D}$-) block
and any blocks before (after) it.
\end{enumerate}
\end{enumerate}

\begin{figure}[h]
    \centering
    \subfloat[IntraBlock N$_X$O]{\includegraphics[scale=.45]{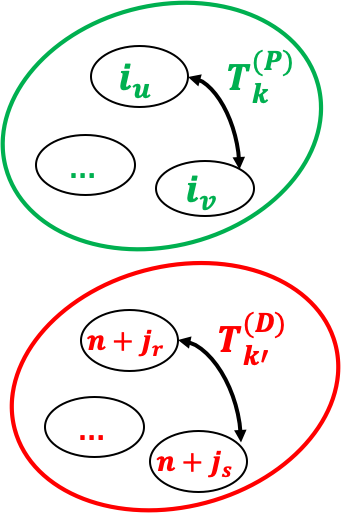}\label{subfig:o_1}}
    \hspace{0.7cm}
    \subfloat[InterBlock N$_X$O]{\includegraphics[scale=.45]{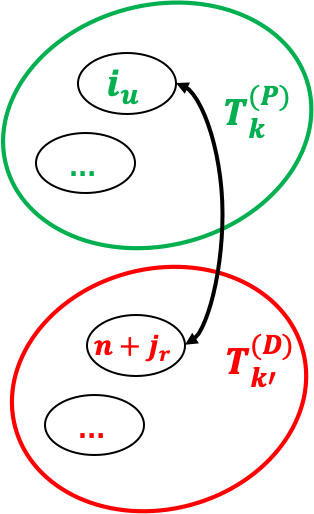}\label{subfig:o_2}}
    
    \subfloat[N$^2_X$O]{\includegraphics[scale=.48]{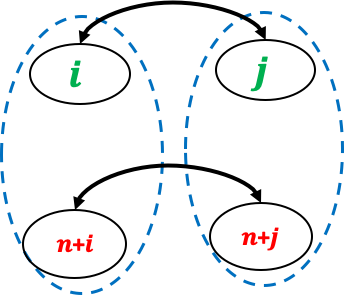}\label{subfig:o_3}}
    \hspace{0.2cm}
    \subfloat[SameB$_X$O]{\includegraphics[scale=.48]{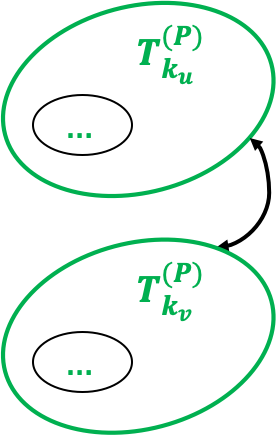}\label{subfig:o_4}}
    \hspace{0.2cm}
    \subfloat[MixB$_X$O]{\includegraphics[scale=.48]{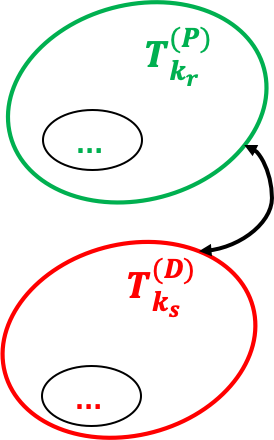}\label{subfig:o_5}}
    \vspace{-0.12in}
    \caption{Five admissible operators}
    \label{fig:5_opt}
    \vspace{-0.02in}
\end{figure}

Below we will elaborate what each operator does, its associated reward in the context of RL, and why it upholds solution feasibility. 
The toy example in Figure \ref{fig:5_node_tour} is used to demonstrate how to perform each operator.

\noindent \textbf{Node-exchange operator (N$_X$O)}\\
\noindent \textbf{(N1) IntraBlock N$_X$O:} Given a tour $\mathcal{T}$ represented by a block sequence $0 \rightarrow \mathcal{T}^{(\mathcal{P})}_{1} \rightarrow \cdots \rightarrow \mathcal{T}^{(\mathcal{P})}_{k} \rightarrow \cdots \rightarrow \mathcal{T}^{(\mathcal{D})}_{k'} \rightarrow \cdots \rightarrow \mathcal{T}^{(\mathcal{D})}_{K} \rightarrow 0$, IntraBlock N$_X$O swaps any two $\mathcal{P}$-nodes $i_u, i_v$ in the $\mathcal{P}$-block $\mathcal{T}^{(\mathcal{P})}_{k}$, i.e., $\forall i_u, i_v \in \mathcal{N}(\mathcal{T}^{(\mathcal{P})}_{k})$. Accordingly, the indices of these two nodes are updated to $p'_{i_u} \leftarrow p_{i_v},\ p'_{i_u} \leftarrow p_{i_v}$ where, $p'_{i_u}, p'_{i_v}$ denote the indices of node $i_u,i_v$ in the new tour $\mathcal{T}'$. 

Similarly, IntraBlock N$_X$O swaps any two $\mathcal{D}$-nodes $n+j_r,n+j_s$ in $\mathcal{D}$-block $\mathcal{T}^{(\mathcal{D})}_{k'}$. Accordingly, their node indices are updated to $p'_{n+j_r} \leftarrow p_{n+j_s},\ p'_{n+j_r} \leftarrow p_{n+j_s}$, where $p'_{n+j_r}, p'_{n+j_s}$ denote the indices of node $n+j_r,n+j_s$ in the new tour $\mathcal{T}'$.

The reward of IntraBlock N$_X$O is the change in travel cost, denoted as $r_{N_1}$, after the execution of the operator. 
We demonstrate how to compute the reward of IntraBlock N$_X$O in the toy example shown in Figure \ref{fig:IntraBlock_operator}. IntraBlock N$_X$O swaps node 2 and 3 in  $\mathcal{P}$-block $\mathcal{T}^{(\mathcal{P})}_1$ and gets a new tour $T':0 \rightarrow 1 \rightarrow 3 \rightarrow 2 \rightarrow 7 \rightarrow 8 \rightarrow 4 \rightarrow 5 \rightarrow 6 \rightarrow 9 \rightarrow 10 \rightarrow 0$. Therefore, the difference between the travel cost of two tours is 
\begin{align}
r_{N_1} =c_{13}+c_{27}-(c_{12}+c_{37})
\end{align}
Note that IntraBlock N$_X$O is like the naive operator except for the condition that the visiting sequence of pickup (delivery) nodes in a $\mathcal{P}$- ($\mathcal{D}$-) block can be randomly swapped.

\begin{prop}\label{prop:inner}
    IntraBlock N$_X$O is a mapping from one feasible solution to another. 
\end{prop}

\begin{figure}[h]
    \centering
    \includegraphics[scale=0.25]{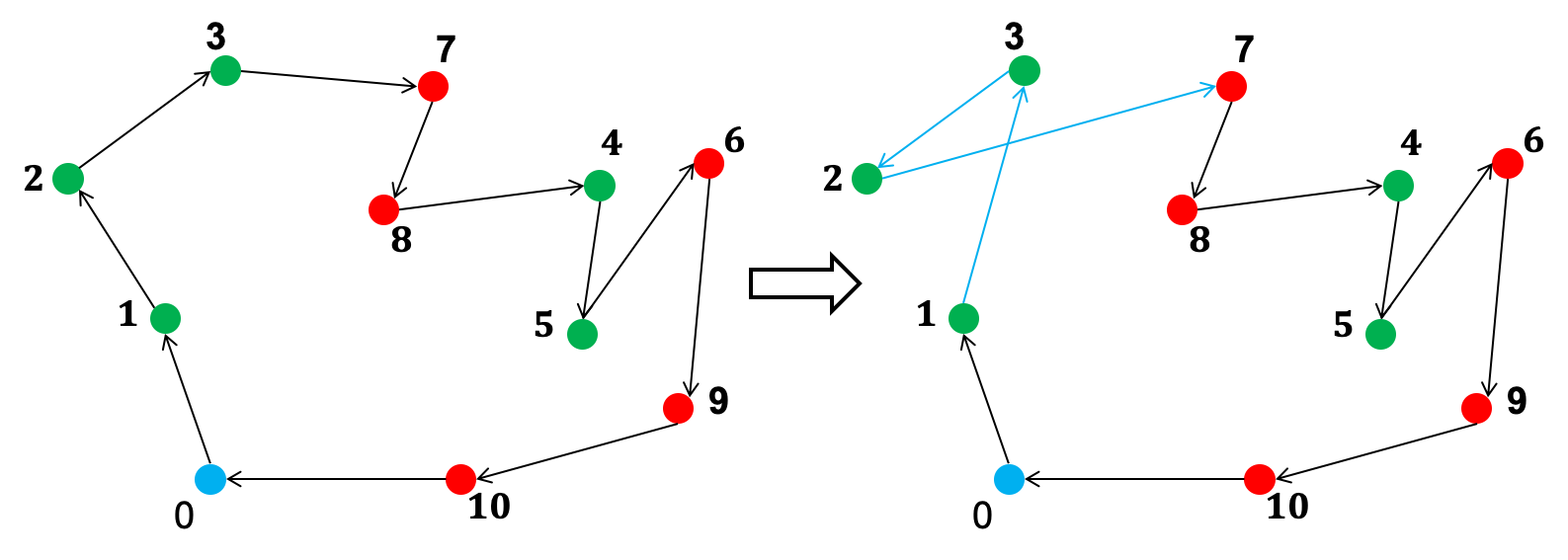}
    \vspace{-0.18in}
    \caption{IntraBlock N$_X$O on the toy example}
    \label{fig:IntraBlock_operator}
    \vspace{-0.1in}
\end{figure}

\noindent \textbf{(N2) InterBlock N$_X$O:} Given a tour $\mathcal{T}$ represented by a block sequence $0 \rightarrow \mathcal{T}^{(\mathcal{P})}_{1} \rightarrow \cdots \rightarrow \mathcal{T}^{(\mathcal{P})}_{k} \rightarrow \cdots \rightarrow
\mathcal{T}^{(\mathcal{D})}_{K} \rightarrow 0$, InterBlock N$_X$O swaps any $\mathcal{P}$-node $i_u$ in $\mathcal{P}$-block $\mathcal{T}^{(\mathcal{P})}_{k}$ and any $\mathcal{D}$-node $n+j_r$ in $\mathcal{D}$-block $\mathcal{T}^{(\mathcal{D})}_{k'} (k>k')$, i.e., $\forall i_u\in \mathcal{N}(\mathcal{T}^{(\mathcal{P})}_{k}), n+j_r \in \mathcal{N}(\mathcal{T}^{(\mathcal{D})}_{k'})$. Accordingly, the indices of these two nodes are updated to $p'_{i_u} \leftarrow d_{n+j_r},\ d'_{n+j_r} \leftarrow p_{i_u}$. $p'_{i_u}$ is the new index of node $i_u$ and $d'_{n+j_r}$ is the new index of node $n+j_r$. 



The reward of InterBlock N$_X$O is the change in travel cost, denoted as $r_{N_2}$, after the execution of the operator. Figure \ref{fig:outer_operator_ex} shows how to perform InterBlock N$_X$O on the toy example. InterBlock N$_X$O swaps node $7$ in the $\mathcal{D}$-block $\mathcal{T}^{(\mathcal{D})}_2$ and $4$ in the $\mathcal{P}$-block $\mathcal{T}^{(\mathcal{P})}_3$. The visiting sequence changes from $3 \rightarrow 7 \rightarrow 8 \rightarrow 4 \rightarrow 5$ to $3 \rightarrow 4 \rightarrow 8 \rightarrow 7 \rightarrow 5$ (marked in light blue). Therefore, the reward of InterBlock N$_X$O is calculated as
\begin{align}
r_{N_2}=l_{34}+l_{75}-(l_{37}+l_{45})
\end{align}
\begin{figure}[h]
    \centering
    \vspace{-0.18in}
    \includegraphics[scale=0.25]{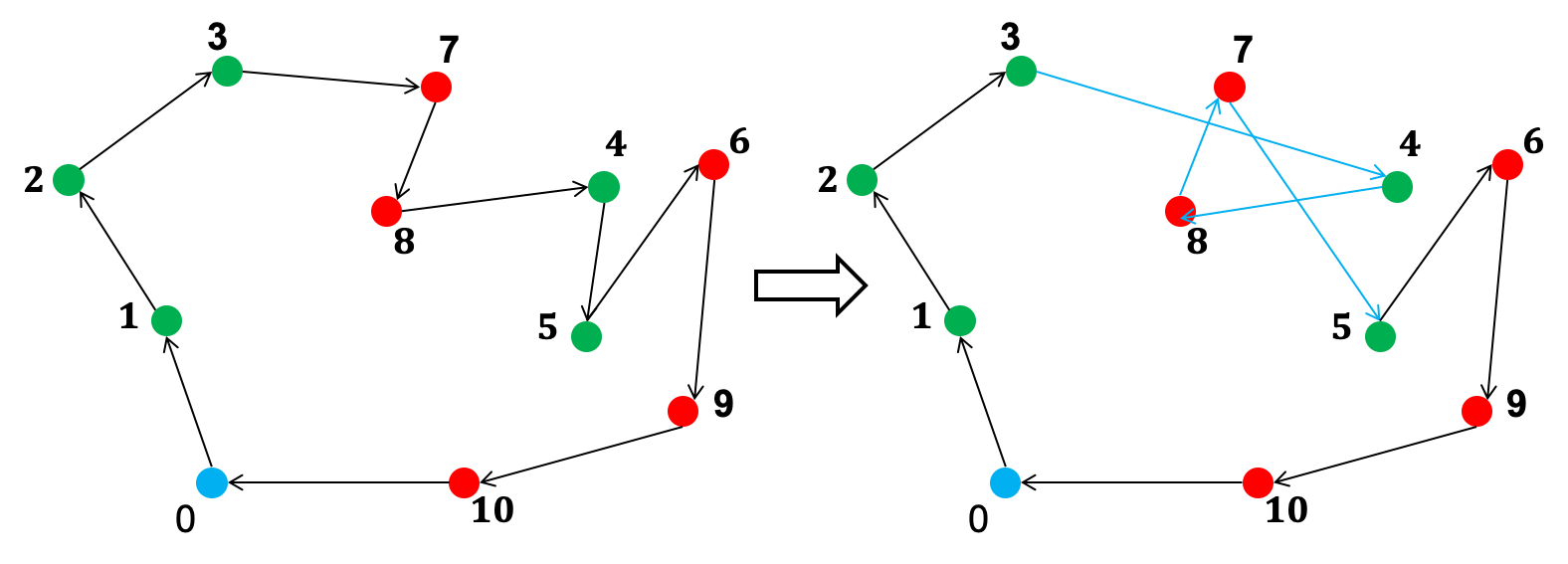}
    \vspace{-0.18in}
    \caption{InterBlock N$_X$O on the toy example}
    \vspace{-0.18in}
    \label{fig:outer_operator_ex}
\end{figure}
\begin{prop}\label{prop:outer}
    InterBlock N$_X$O is a mapping from one feasible solution to another. 
\end{prop}
The proof is in Appendix \ref{appendix:proof}.

\noindent \textbf{(N3) N$^2_X$O:} Given a tour $\mathcal{T}$, for any two PD node pairs $(i,n+i)$ and $(j,n+j)$, N$^2_X$O swaps $\mathcal{P}$-nodes $i, j$ and their respective $\mathcal{D}$-nodes $n+i, n+j$. Accordingly, the indices of two $\mathcal{P}$-nodes are updated to $p'_{i} \leftarrow p_j,\ p'_{j} \leftarrow p_{i}$ and the indices of two $\mathcal{D}$-nodes are updated to $d'_{n+i} \leftarrow d_{n+j},\ d'_{n+j} \leftarrow d_{n+i}$, respectively.

The reward of InterBlock N$^2_X$O is the change in travel cost, denoted as $r_{N_3}$, after the execution of the operator. Figure \ref{fig:pd_operator} shows how to perform N$^2_X$O on the toy example. The reward of N$^2_X$O is 
\begin{align}
r_{N_3}&=l_{02}+l_{13}+l_{36}+l_{68}+l_{57}+l_{79} \nonumber \\
&-(l_{01}+l_{23}+l_{37}+l_{78}+l_{56}+l_{69})
\end{align}
\begin{figure}[h]
    \centering
    \vspace{-0.18in}
    \includegraphics[scale=0.25]{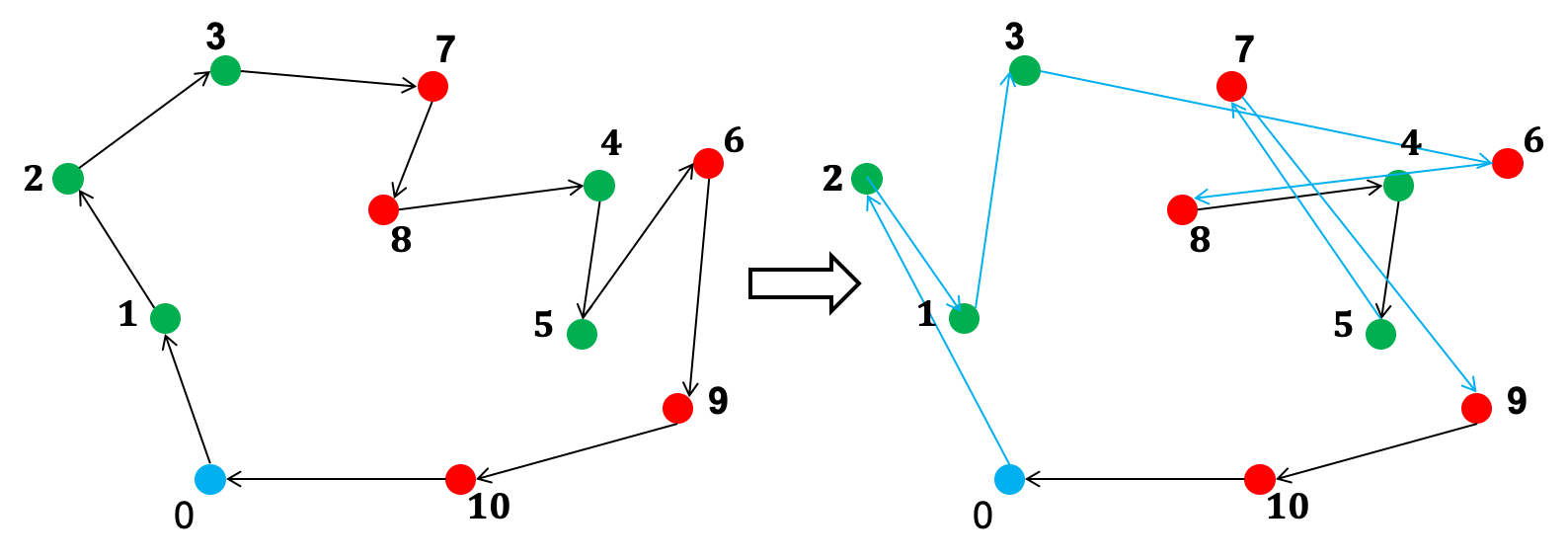}
    \vspace{-0.18in}
    \caption{N$^2_X$O on the toy example}
    \label{fig:pd_operator}
    \vspace{-0.18in}
\end{figure}
\begin{prop}
   N$^2_X$O is a mapping from one feasible solution to another. 
\end{prop}
The proof is in Appendix \ref{appendix:proof}.

Note that operators developed in some literature actually perform the same function as our proposed node exchange operator (N$_X$O). Below we will demonstrate the equivalence between our N$^{2}_X$O and operators developed from two related work.

In \cite{carrabs2007operator}, a sub-sequence operators, consisting of only PD node pairs, permits the swap of any two PD node pairs within this sequence. The definition of this operator in \cite{carrabs2007operator} can be demonstrated via the tour example shown in Figure \ref{fig:pd_operator}.   
In this tour, there exists only one such sub-sequence $0  \rightarrow1 \rightarrow 2\rightarrow...\rightarrow 9 \rightarrow 10  \rightarrow 0$, encompassing five PD node pairs. Once this sub-sequence is identified, any two PD node pairs among these five PD node pairs can be swapped, which is essentially executing our N$^{2}_X$O. 
\begin{figure}[h]
    \centering
    \vspace{-0.1in}
    \includegraphics[scale=0.28]{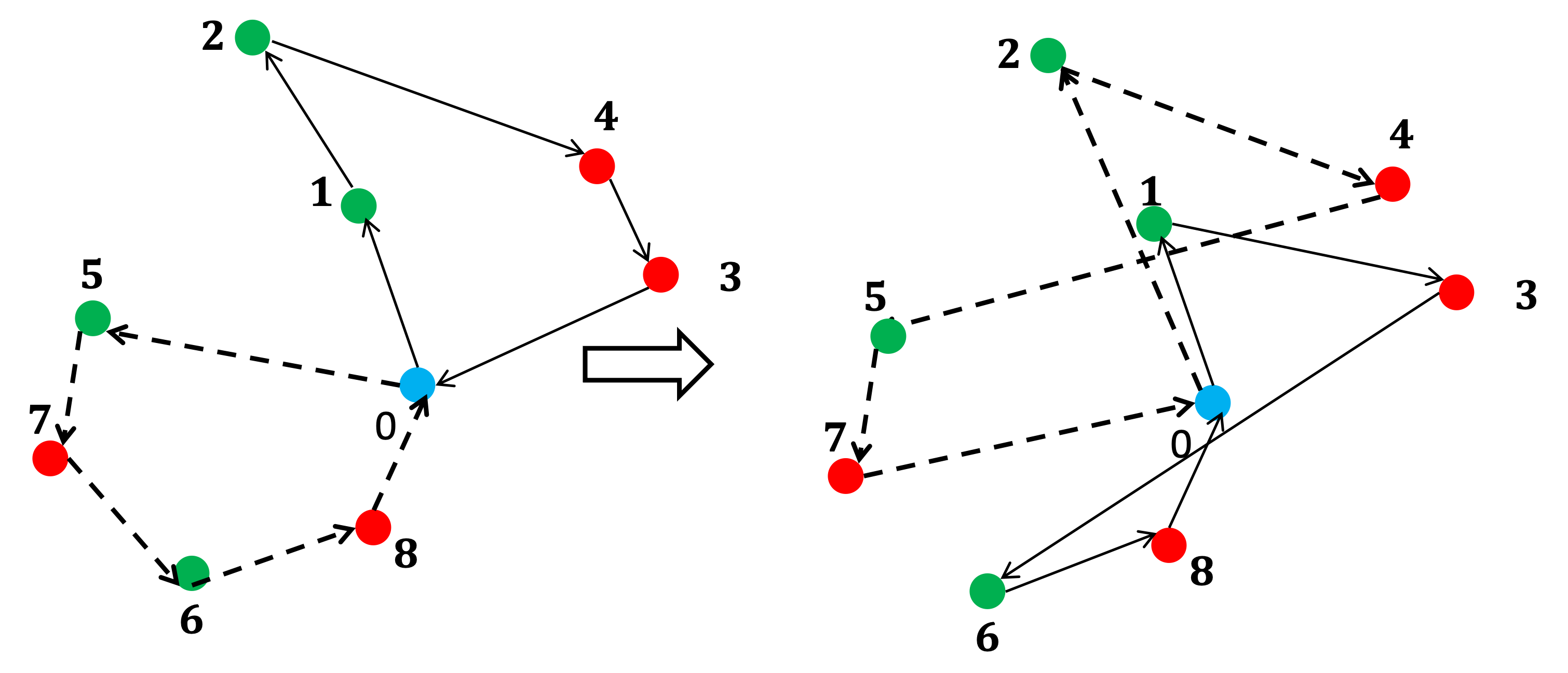}
    \vspace{-0.18in}
    \caption{N$^2_X$O with multiple vehicles}
    \label{fig:pd_operator_intra}
    \vspace{-0.1in}
\end{figure}

In problems with multi-vehicles with finite capacity (Capacitated-PDTSPs): N$^{2}_X$O can be extended to swap passengers (i.e., PD node pairs) carried by different vehicles. The number of node pairs swapped is bounded by the capacity. We demonstrate how to execute N$^{2}_X$O among two vehicles in Figure \ref{fig:pd_operator_intra}. The feasible tours of two vehicles starting from the depot are : $0\rightarrow 1 \rightarrow 2 \rightarrow 4 \rightarrow 3 \rightarrow 0$ and $0\rightarrow 5 \rightarrow 7 \rightarrow 6 \rightarrow 8 \rightarrow 0$. We swap PD node pairs $(2,4)$ with $(6,8)$ and get two new feasible tours for vehicles. It shows that N$^{2}_X$O can be executed across tours belonging to multiple vehicles without violating the precedence constraints. Proposition \ref{prop:outer} still holds. For simplicity, we omit the proof. 

In another study~\cite{ma2022pd}, an insertion operator is proposed, which is defined below. 
\begin{defn} \textbf{Insertion operator~\cite{ma2022pd}:}
    Given a PD node pair $(i,n+i)$, the insertion operator places $i$ and $n+i$ on two positions in a tour formed by nodes $\mathcal{N/}\{i,n+i\}$. Denote the positions of $i$ and $n+i$ in the new tour as $p'_i$ and $d'_{n+i}$, respectively. We have $p'_i<d'_{n+i}$. 
\end{defn}

We use Figure \ref{fig:insert_operator} to explain the insertion operator. A tour $ 0\rightarrow 2 \rightarrow ... \rightarrow 10  \rightarrow 0$  is formed by four PD node pairs, where PD node pair $(1,6)$ is not included. The operator then inserts node 1 between node 0 and 2, node 6 between node 5 and 9, and obtains a new tour $0  \rightarrow1 \rightarrow 2\rightarrow...\rightarrow 9 \rightarrow 10  \rightarrow 0$. 

\begin{prop}\label{prop:equiv}
The insertion operator is equivalent to N$_X$O (N1, N2, N3) for any node pairs. 
\end{prop}

The proof is in Appendix \ref{appendix:proof}. Proposition \ref{prop:equiv} demonstrates the equivalence between the insertion operator and our proposed node exchange operators defined in N$_X$O.  

\begin{figure}[h]
    \centering
    \vspace{-0.1in}
    \includegraphics[scale=0.3]{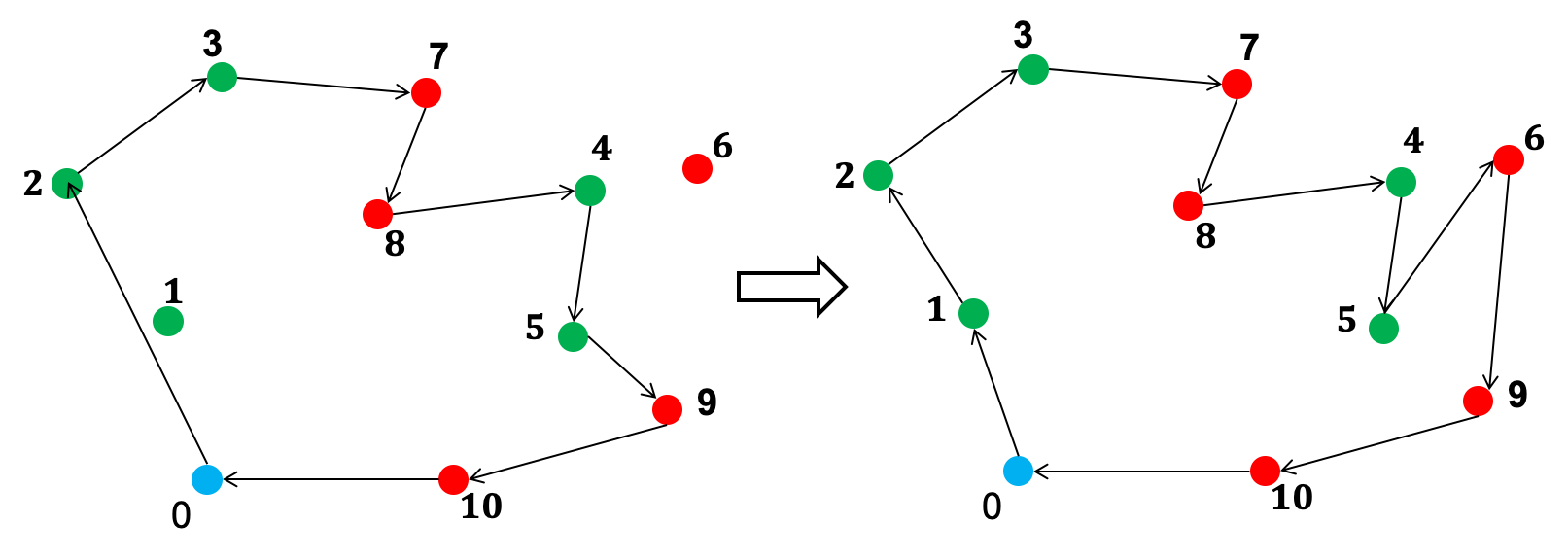}
    \vspace{-0.18in}
    \caption{Insertion operator on the toy example}
    \label{fig:insert_operator}
    \vspace{-0.1in}
\end{figure}


\noindent \textbf{Block-exchange operator (B$_X$O)}


\noindent \textbf{(B1) SameB$_X$O:} Given a tour $\mathcal{T}$ represented by a block sequence $0 \rightarrow \mathcal{T}^{(\mathcal{P})}_{1} \rightarrow \cdots \rightarrow \mathcal{T}^{(\mathcal{D})}_{K} \rightarrow 0$, SameB$_X$O swaps any two $\mathcal{P}$-blocks from a sequence of $\mathcal{P}$-blocks $\mathcal{T}^{(\mathcal{P})}_{k_1} \rightarrow \cdots \rightarrow \mathcal{T}^{(\mathcal{P})}_{k_h}$, i.e., $\forall k_u, k_v \in \{k_1,\cdots k_h \}$. Accordingly, the indices of two $\mathcal{P}$-blocks are updated to $k'_u \leftarrow k_v,\  k'_v \leftarrow k_u$, where $k'_u, k'_v$ denote the indices of blocks $\mathcal{T}^{(\mathcal{P})}_{k_u}, \mathcal{T}^{(\mathcal{P})}_{k_v}$ in the new tour $\mathcal{T}'$. Similarly, for any two $\mathcal{D}$-blocks in a sequence of $\mathcal{D}$-blocks $\mathcal{T}^{(\mathcal{D})}_{b_1} \rightarrow \cdots \rightarrow \mathcal{T}^{(\mathcal{D})}_{b_m}$, i.e., $\forall b_r, b_s \in \{b_1,\cdots b_m \}$, SameB$_X$O swaps them. Accordingly, the indices of two $\mathcal{D}$-blocks are updated to $b'_r \leftarrow b_s,\ b'_s \leftarrow b_r$, where $b'_r, b'_s$ denote the indices of blocks $\mathcal{T}^{(\mathcal{D})}_{b_r}, \mathcal{T}^{(\mathcal{D})}_{b_s}$ in the new tour.


The reward of SameB$_X$O is the change in travel cost, denoted as $r_{B_1}$, after the execution of the operator. Figure \ref{fig:sameblock} shows how to perform SameB$_X$O on the toy example. The tour is represented by a block sequence $0 \rightarrow \mathcal{T}^{(\mathcal{P})}_1 \rightarrow \mathcal{T}^{(\mathcal{D})}_2 \rightarrow \mathcal{T}^{(\mathcal{P})}_3 \rightarrow \mathcal{T}^{(\mathcal{D})}_4 \rightarrow \mathcal{T}^{(\mathcal{D})}_5 \rightarrow 0$. Two $\mathcal{D}$-blocks $\mathcal{T}^{(\mathcal{D})}_4$ and $\mathcal{T}^{(\mathcal{D})}_5$ are swapped. The tour is changed to $0 \rightarrow \mathcal{T}^{(\mathcal{P})}_1 \rightarrow \mathcal{T}^{(\mathcal{D})}_2 \rightarrow \mathcal{T}^{(\mathcal{P})}_3 \rightarrow \mathcal{T}^{(\mathcal{D})}_5 \rightarrow \mathcal{T}^{(\mathcal{D})}_4 \rightarrow 0$ (marked in light blue). Therefore, the reward is 
\begin{align}
    r_{B_1}=l_{5,10}+l_{90}-(l_{56}+l_{10,0})
\end{align}
\begin{figure}[h]
    \centering
    \vspace{-0.18in}
    \includegraphics[scale=.23]{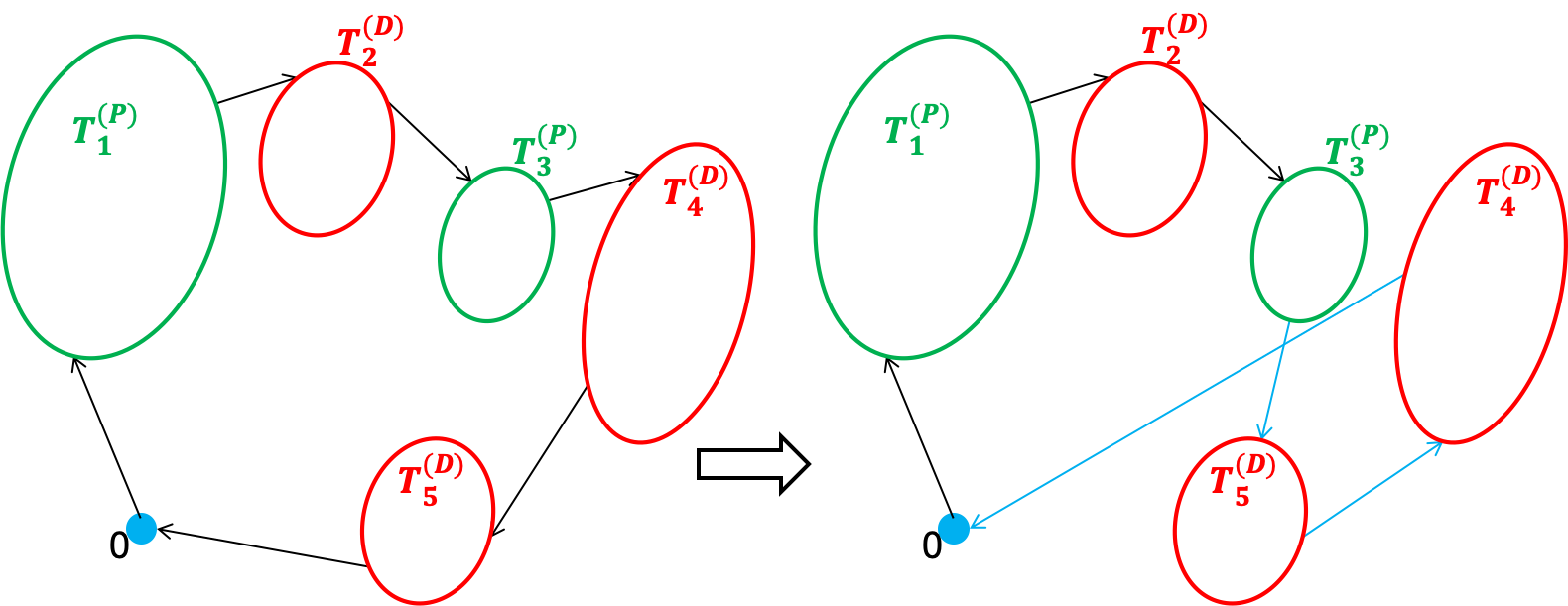}
    \vspace{-0.1in}
    \caption{SameB$_X$O on the toy example}
    \label{fig:sameblock}
    \vspace{-0.15in}
\end{figure}
\begin{prop}
    SameB$_X$O is a mapping from one feasible solution to another. 
\end{prop}
The proof is in Appendix \ref{appendix:proof}.

\noindent \textbf{(B2) MixB$_X$O:} Given a block sequence $\mathcal{T}: 0 \rightarrow \mathcal{T}^{(\mathcal{P})}_{1} \rightarrow \cdots \rightarrow \mathcal{T}^{(\mathcal{D})}_{K} \rightarrow 0$, MixB$_X$O swaps a $\mathcal{P}$-block $\mathcal{T}^{(\mathcal{P})}_{k_r}$ and a $\mathcal{D}$-block $\mathcal{T}^{(\mathcal{D})}_{k_s}, \forall k_r>k_s$. Accordingly, the indices of these two blocks are updated to $k'_r \leftarrow k_s,\  k'_s \leftarrow k_r$, respectively, where $k'_r,k'_s$ denote the indices of $\mathcal{P}$-block $\mathcal{T}^{(\mathcal{P})}_{k_r}$ and $\mathcal{D}$-block $\mathcal{T}^{(\mathcal{D})}_{k_s}$ in the new tour.

The reward of MixB$_X$O is the change in travel cost, denoted as $r_{B_2}$, after the execution of the operator. In Figure \ref{fig:mixedblock}, $\mathcal{T}^{(\mathcal{P})}_{3}$ and $\mathcal{T}^{(\mathcal{D})}_{2}$ are swapped. The reward of MixB$_X$O is calculated as:
\begin{align}
    r_{B_2} =l_{34}+l_{86}-(l_{37}+l_{56})
\end{align}

\begin{figure}[h]
    \centering
    \vspace{-0.12in}
    \includegraphics[scale=0.23]{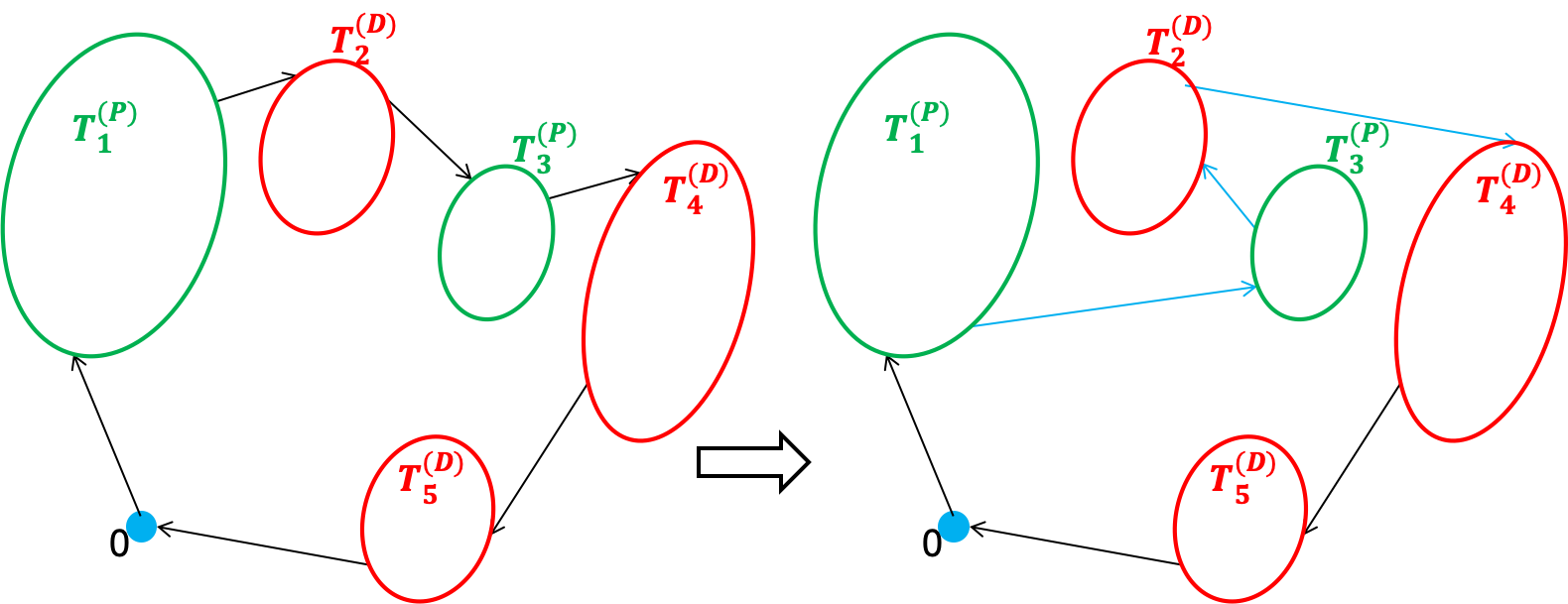}
    \vspace{-0.1in}
    \caption{MixB$_X$O on the toy example}
    \label{fig:mixedblock}
    \vspace{-0.12in}
\end{figure}

\begin{figure*}[t]
   \centering
   \vspace{-0.12in}
   \includegraphics[scale=.35]{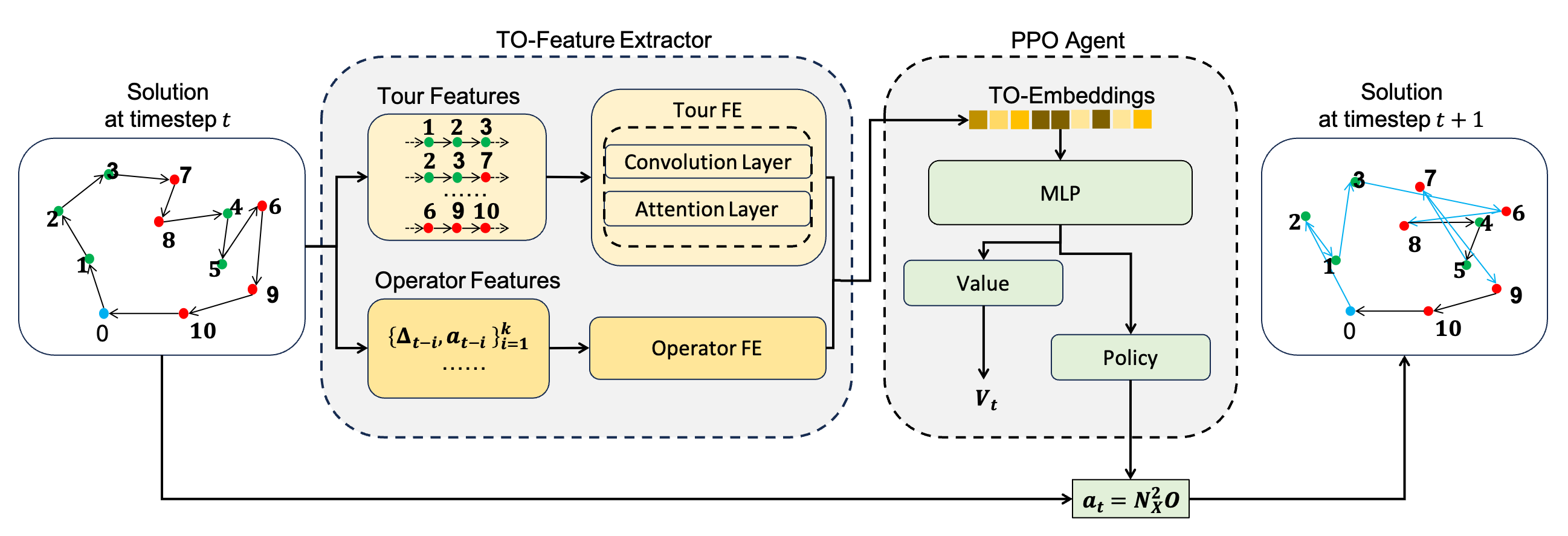}
   \vspace{-0.15in}
    \caption{An illustration of L2T policy network}
    \label{fig:flow_appendix}
    \vspace{-0.1in}
\end{figure*}

\begin{prop}
    MixB$_X$O is a mapping from one feasible solution to another. 
\end{prop}
The proof is in Appendix \ref{appendix:proof}.

\section{Solution Approach}
\label{sec:sol}
We propose Algorithm \ref{alg:l2t} (aka. L2T) to solve PDTSP. We first initialize policy and value networks, which are parameterized by $\theta$ and $\phi$, respectively. The input to the policy network is a state (i.e., a tour), and the output is policy distribution for each operator introduced in Section \ref{sec:opt_design}. During the $w$th episode of the training process (Line 3 in Algorithm \ref{alg:l2t}), we begin by constructing a initial tour (i.e., initial state $s$). In Line 5, we use policy networks to generate a sequence of operators for the initial tour until the terminal step $M$. We then execute the operators and transform the tour into new tours, thereby triggering a state transition from $s$ to $s'$. We calculate the reward corresponding to each operator. The tuple $(s,a,r,s')$ is stored in the replay buffer. Lines 11-13 update the policy and value networks. Lines 14 calculates the minimum tour cost as $C_{w+1}=\min_{\mathcal{T} \in \mathbf{s}} C(\mathcal{T})$. Line 15 checks convergence by computing the difference in the minimum tour cost between two episodes. The convergence condition is $|C_{w+1}-C_{w}|<\epsilon$. Note that $C$ is the travel cost of a tour, while the value function $V$ is the cumulative reward corresponding to the total change in tour cost from that of the initial tour. 

Figure \ref{fig:flow_appendix} shows the structure of our policy network. The policy network employs a feature extractor module, which extracts the salient features of tours and operator-specific information. For a given tour, we extract a feature matrix $\in \mathbb{R}^{|\mathcal{N}|\times 12}$, which is composed of the locations, type (pickup or delivery),  preceding and succeeding nodes, and distance between two adjacent nodes. As for the operator, the features derived are the improvement of the tour, the difference between the current tour and the best known tour, the number of steps without improvement, and $N$ recent historical operators applied and their improvements, which results in a vector $\in \mathbb{R}^{2N+3}$. With the raw features in hand, we leverage a neural architecture to process the features. During the forward pass, the raw tour features are passed through a sequence of 1x1 convolution layers and enhanced by self multi-head attention. The resultant embeddings are concatenated with the operator features to generate the final embeddings used by downstream policy. By default, Feature Extractor uses a hidden dimension of 256, 4 heads for multi-head self attention. ReLU activation is applied. As for the network architecture of the PPO agent, the processed features are fed into an MLP which consists of two sequential layers, each of dimension as 256. The agent then branches out to two separate linear layers to output value and action probabilities. 

\begin{algorithm}[h]
		\caption{L2T}
		\label{alg:l2t}
		\begin{algorithmic}[1]
			\State Initialize policy network parameterized by $\theta$; value network parameterized by $\phi$; replay buffer; 
			\For{$w\leftarrow 0$ to $W$}
                \State Initialize state $s$ (i.e., tour construction);
                \While{not at terminal step $M$}
                \State Generate action (i.e., operator) from $\pi_{\theta_w}(s)$;
                \State Execute operators for the tour;
                \State Obtain reward $r$ and new state $s'$
                \State Store $(s,a,r,s')$ into buffer;
                \EndWhile
                \State Sample a batch of experience from buffer;
                \State Compute advantages: \small{$\frac{1}{\mathcal{K}}\sum^{\mathcal{K}}_{l=1}[r+V_{\phi_w}(s'_{l})-V_{\phi_w}(s_l)]$}
                \State Update policy and value network;
                \State Obtain the minimum tour cost $C_{w+1}$;
                \State Check convergence: $|C_{w+1}-C_{w}|<\epsilon$.
			\EndFor
		\end{algorithmic}
\end{algorithm}

Baselines include:
\begin{enumerate}
    \item Google OR tools: Google OR tools is an open-source, fast and portable software suite for solving combinatorial optimization problems \cite{googleor}.
    \item Gurobi: Gurobi is a mathematical optimization solver, widely used for solving large-scale linear, quadratic, and mixed-integer programming problems \cite{gurobi}.
    \item Pointer Network (Ptr-Net): The Ptr-Net is a sequence-to-sequence model for TSPs. We incorporate the Ptr-Net into a RL framework \cite{bello2016rl} to solve PDTSP by iteratively checking the feasibility of the route at each node.
    \item Transformer: Transformer-based architecture modifies the Ptr-Net via an attention mechanism \cite{kool2018attention}. We replace the Ptr-Net in RL with a transformer to solve PDTSP.
    \item LKH3.0: It is a famous heuristic algorithm for solving TSPs and vehicle routing problems (VRPs) \cite{lkh3}. Note that LKH1.0 cannot solve PDTSP. We thus use LKH3.0.
    \item L2T-Naive: We replace our proposed operators with the naive operator defined in Section \ref{sec:pre} and implement the RL algorithm (L2T). 
    \item L2T (Individual operator): To demonstrate the performance of our unified operator set, we simply run each individual operator and implement the RL algorithm (L2T), including L2T-N1, L2T-N2, L2T-N3, L2T-B1, L2T-B2 and L2T-insertion.
\end{enumerate}

\section{Experiment}
\label{sec:numer}
In this section, we present our experiment results. We first introduce the experiment setup. 

\noindent \textbf{Experiment setup:} PDTSPs with five problem scales are considered, namely \( n = 5, 10, 20, 30, 50 \), resulting in node sets \( |\mathcal{N}| \) of sizes 11, 21, 41, 61, and 101, respectively. We employ PDTSP instances sourced from the Grubhub Test Instance Library \cite{lkh3}, where the location of each node, as well as the depot, is uniformly sampled from a unit square. The traveling cost between two nodes, $c_{ij}$, is the corresponding Euclidean distance. All computational experiments are conducted on a Google Cloud Platform instance of type n1-highmem-8, equipped with two NVIDIA T4 GPUs.



\noindent \textbf{Numerical results:} Figure \ref{fig:convergence} visualizes the convergence performance of our algorithm. The x-axis is the episode, and the y-axis denotes the convergence gap for instances with $|\mathcal{N}|=11,21,41,61,101$, respectively. In comparison to instances with $|\mathcal{N}|=11,21,41$, it takes more episodes for the minimum cost with $|\mathcal{N}|=61,101$ to converge. More visualizations are in Appendix.


In Table 1, we compare the average minimum travel cost obtained by our method with the baselines. The results indicate that our method outperforms the baselines in terms of average travel cost. Particularly, the average distance achieved by L2T is significantly shorter (around 10.3\%) than that achieved by Google OR tools and Gurobi when $|\mathcal{N}|\geqslant 41$. Our algorithm achieves the minimum travel cost among all methods when $|\mathcal{N}|=61,101$. It is shown that the Ptr-Net and transformer fail to find feasible solutions within a reasonable time when $|\mathcal{N}|\geqslant 21$. 
This is because the end-to-end approach using sequence-to-sequence models outputs infeasible Hamiltonian cycles and fails to capture the precedence constraints. Table 2 shows the training time when learning methods converge. The training time of L2T-naive significantly increases as the problem scale grows larger. This is due to the fact that L2T-naive must explore a space containing a large number of infeasible Hamiltonian cycles (Proposition \ref{prop:num_tour}), which makes the training process time-consuming. In comparison to our method (L2T), L2T-B1, L2T-B2 achieve a lower training time. However, these operators output much worse solutions (Table 1). This is because individual operators perform limited functions. Block operators (B1, B2) cannot improve solutions anymore by swapping nodes within a block. They converge quickly but lose some solution quality. L2T-N3 achieves similar solution quality but has higher computational cost. This is because it cannot swiftly exchange a batch of nodes as block operators do. It takes longer to converge for achieving the same solution quality. We list the training time of the cutting-edge method \cite{ma2022pd}: 1 day for N=21, 3 days for N=51, 7 days for N=101. Other baselines go beyond 10 days. They utilize insertion operators similar to the individual operator L2T-insertion. Table 2 shows that the training time of L2T-insertion is 8 times more than others.

For Capacitated-PDTSPs, we execute operators across tours belonging to multiple vehicles (see Figure \ref{fig:pd_operator_intra}). we compared our method with OR-Tools on an instance with 500 nodes and a maximum of 30 vehicles. We set a time limit of 12 hours. Our method found a solution with cost 22.165, while OR-Tools found a solution with cost 24.546.

\begin{figure}[ht]
    \centering
   \includegraphics[scale=.3]{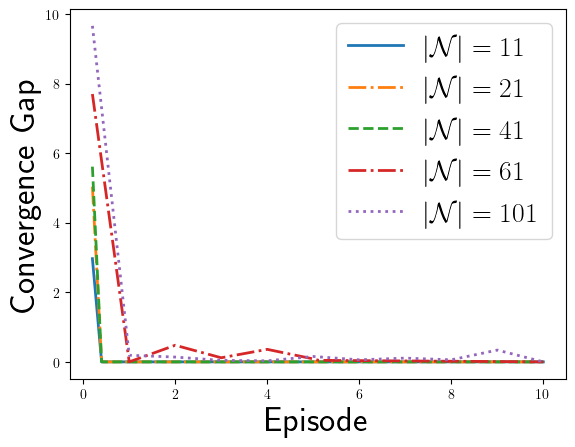}
   \vspace{-0.15in}
    \caption{Convergence performance}
    \label{fig:convergence}
    \vspace{-0.15in}
\end{figure}

\begin{table}[ht]
    \footnotesize
    \centering
    \begin{tabular}{c|c c c c c c}
        \hline
        & \multicolumn{6}{c}{$|\mathcal{N}|$} \\
    \cline{2-7}
        & $11$ & $21$ & $41$ & $61$ & $101$ & $501$ \\
        \hline
        Google OR Tools & 3.124 & 4.729 & 6.982 & 7.603 & 10.307 & 21.901 \\
        Gurobi & 3.031 & 4.555 & 6.779 & 7.574 & - & - \\
        Ptr-Net+RL & 4.298 & - & - & - & - & - \\
        Transformer+RL & 3.965 & - & - & - & - & - \\
        LKH3.0 & 3.031 & 4.555 & 6.311 & 7.195 & 9.255 & \textbf{20.277} \\
        L2T-Naive & 3.031 & 4.555 & 6.311 & 7.214 & 10.321 & - \\
        L2T-N1 & 3.031 & 4.555 & 6.382 & 7.241 & 10.315 & - \\
        L2T-N2 & 3.031 & 4.556 & 6.944 & 8.284 & 11.428 & - \\
        L2T-N3 & 3.031 & 4.555 & 6.358 & 7.194 & 9.504 & - \\
        L2T-B1 & 3.031 & 4.602 & 6.673 & 7.743 & 10.516 & - \\
        L2T-B2 & 3.031 & 4.600 & 6.719 & 7.621 & 10.759 & - \\
        L2T-Insertion & 3.031 & 4.555 & 6.311 & 7.245 & 10.519 & - \\
        L2T (Ours) & \textbf{3.031} & \textbf{4.555} & \textbf{6.311} & \textbf{7.189} & \textbf{9.242} & 20.591 \\
        \hline
    \end{tabular}
    \caption{Comparison of cost on PDTSP instances}
    \vspace{-0.15in}
    \label{tab:cost_comp}
\end{table}

\begin{table}[ht]
    \footnotesize
    \centering
    \begin{tabular}{c|c c c c c}
    \hline
    & \multicolumn{5}{c}{$|\mathcal{N}|$} \\
    \cline{2-6}
        & $11$ & $21$ & $41$ & $61$ & $101$ \\
        \hline
         L2T-Naive & 282.9 & 480.5 & 1906.0 & 6093.5 & 30634.8 \\
         L2T-N1 & 219.6 & 398.5 & 837.9 & 1881.3 & 7353.8 \\
         L2T-N2 & 281.3 & 362.9 & 532.4 & 789.9 & 1786.8 \\
         L2T-N3 & 278.5 & 397.6 & 852.7 & 2080.4 & 7296.8 \\
         L2T-B1 & 291.5 & 380.0 & 605.8 & 931.6 & 1968.9 \\
         L2T-B2 & 283.0 & 396.2 & 604.5 & 930.9 & 2099.1\\
         L2T-Insertion & 281.3 & 444.5 & 2827.2 & 9597.6 & 34280.0 \\
         L2T (Ours) & 270.8 & 358.7 & 575.1 & 1100.6 & 4459.6 \\
         \hline
    \end{tabular}
    \caption{Training time (s)}
    \label{tab:time_comp}
\end{table}
\vspace{-0.2cm}

\section{Conclusion}
\label{sec:conclu}

This paper leverages a unified learning operator set in an RL framework to solve PDTSP. By confining the search space to feasible solutions, we have effectively mitigated the computational challenges raised by precedence constraints in PDTSP. Our learning operators, designed to preserve solution feasibility, have demonstrated superior computational efficiency and solution quality across different problem scales. The proposed approach empowers a range of real-world applications in transportation, mobility, and logistics.

\bibliographystyle{ACM-Reference-Format}
\bibliography{aamas24}


\begin{thebibliography}{36}


\ifx \showCODEN    \undefined \def \showCODEN     #1{\unskip}     \fi
\ifx \showDOI      \undefined \def \showDOI       #1{#1}\fi
\ifx \showISBNx    \undefined \def \showISBNx     #1{\unskip}     \fi
\ifx \showISBNxiii \undefined \def \showISBNxiii  #1{\unskip}     \fi
\ifx \showISSN     \undefined \def \showISSN      #1{\unskip}     \fi
\ifx \showLCCN     \undefined \def \showLCCN      #1{\unskip}     \fi
\ifx \shownote     \undefined \def \shownote      #1{#1}          \fi
\ifx \showarticletitle \undefined \def \showarticletitle #1{#1}   \fi
\ifx \showURL      \undefined \def \showURL       {\relax}        \fi
\providecommand\bibfield[2]{#2}
\providecommand\bibinfo[2]{#2}
\providecommand\natexlab[1]{#1}
\providecommand\showeprint[2][]{arXiv:#2}

\bibitem[Barrett et~al\mbox{.}(2020)]%
        {barrett2020rl}
\bibfield{author}{\bibinfo{person}{Thomas Barrett}, \bibinfo{person}{William Clements}, \bibinfo{person}{Jakob Foerster}, {and} \bibinfo{person}{Alex Lvovsky}.} \bibinfo{year}{2020}\natexlab{}.
\newblock \showarticletitle{Exploratory Combinatorial Optimization with Reinforcement Learning}.
\newblock \bibinfo{journal}{\emph{Proceedings of the AAAI Conference on Artificial Intelligence}}  \bibinfo{volume}{34} (\bibinfo{date}{04} \bibinfo{year}{2020}), \bibinfo{pages}{3243--3250}.
\newblock


\bibitem[Bello et~al\mbox{.}(2017)]%
        {bello2016rl}
\bibfield{author}{\bibinfo{person}{Irwan Bello}, \bibinfo{person}{Hieu Pham}, \bibinfo{person}{Quoc~V. Le}, \bibinfo{person}{Mohammad Norouzi}, {and} \bibinfo{person}{Samy Bengio}.} \bibinfo{year}{2017}\natexlab{}.
\newblock \bibinfo{booktitle}{\emph{Neural Combinatorial Optimization with Reinforcement Learning}}.
\newblock


\bibitem[Carrabs et~al\mbox{.}(2007)]%
        {carrabs2007operator}
\bibfield{author}{\bibinfo{person}{Francesco Carrabs}, \bibinfo{person}{Jean-FranÃ§ois Cordeau}, {and} \bibinfo{person}{Gilbert Laporte}.} \bibinfo{year}{2007}\natexlab{}.
\newblock \showarticletitle{Variable Neighborhood Search for the Pickup and Delivery Traveling Salesman Problem with LIFO Loading}.
\newblock \bibinfo{journal}{\emph{Informs Journal on Computing}}  \bibinfo{volume}{19} (\bibinfo{date}{11} \bibinfo{year}{2007}), \bibinfo{pages}{618--632}.
\newblock
\urldef\tempurl%
\url{https://doi.org/10.1287/ijoc.1060.0202}
\showDOI{\tempurl}


\bibitem[Chen and Tian(2019)]%
        {chen2019encode}
\bibfield{author}{\bibinfo{person}{Xinyun Chen} {and} \bibinfo{person}{Yuandong Tian}.} \bibinfo{year}{2019}\natexlab{}.
\newblock \showarticletitle{Learning to Perform Local Rewriting for Combinatorial Optimization}. In \bibinfo{booktitle}{\emph{Proceedings of the 33rd International Conference on Neural Information Processing Systems}}.
\newblock


\bibitem[Costa et~al\mbox{.}(2020)]%
        {Costa2020rl}
\bibfield{author}{\bibinfo{person}{Paulo Costa}, \bibinfo{person}{Jason Rhuggenaath}, \bibinfo{person}{Yingqian Zhang}, {and} \bibinfo{person}{Alp~Eren Akçay}.} \bibinfo{year}{2020}\natexlab{}.
\newblock \showarticletitle{Learning 2-opt Heuristics for the Traveling Salesman Problem via Deep Reinforcement Learning}. In \bibinfo{booktitle}{\emph{Asian Conference on Machine Learning}}.
\newblock


\bibitem[Dai et~al\mbox{.}(2017)]%
        {dai2017rl}
\bibfield{author}{\bibinfo{person}{Hanjun Dai}, \bibinfo{person}{Elias~B. Khalil}, \bibinfo{person}{Yuyu Zhang}, \bibinfo{person}{Bistra Dilkina}, {and} \bibinfo{person}{Le Song}.} \bibinfo{year}{2017}\natexlab{}.
\newblock \showarticletitle{Learning Combinatorial Optimization Algorithms over Graphs} \emph{(\bibinfo{series}{NIPS'17})}. \bibinfo{pages}{6351–6361}.
\newblock


\bibitem[Deudon et~al\mbox{.}(2018)]%
        {deudon2018tsp}
\bibfield{author}{\bibinfo{person}{Michel Deudon}, \bibinfo{person}{Pierre Cournut}, \bibinfo{person}{Alexandre Lacoste}, \bibinfo{person}{Yossiri Adulyasak}, {and} \bibinfo{person}{Louis-Martin Rousseau}.} \bibinfo{year}{2018}\natexlab{}.
\newblock \showarticletitle{Learning Heuristics for the TSP by Policy Gradient}. In \bibinfo{booktitle}{\emph{Integration of Constraint Programming, Artificial Intelligence, and Operations Research}}, \bibfield{editor}{\bibinfo{person}{Willem-Jan van Hoeve}} (Ed.). \bibinfo{publisher}{Springer International Publishing}, \bibinfo{address}{Cham}, \bibinfo{pages}{170--181}.
\newblock


\bibitem[Duan et~al\mbox{.}(2020)]%
        {Duan2020EfficientlyST}
\bibfield{author}{\bibinfo{person}{Lu Duan}, \bibinfo{person}{Yang Zhan}, \bibinfo{person}{Haoyuan Hu}, \bibinfo{person}{Yu Gong}, \bibinfo{person}{Jiangwen Wei}, \bibinfo{person}{Xiaodong Zhang}, {and} \bibinfo{person}{Yinghui Xu}.} \bibinfo{year}{2020}\natexlab{}.
\newblock \showarticletitle{Efficiently Solving the Practical Vehicle Routing Problem: A Novel Joint Learning Approach}.
\newblock \bibinfo{journal}{\emph{Proceedings of the 26th ACM SIGKDD International Conference on Knowledge Discovery \& Data Mining}} (\bibinfo{year}{2020}).
\newblock


\bibitem[Dumitrescu et~al\mbox{.}(2009)]%
        {dumi2009tsp}
\bibfield{author}{\bibinfo{person}{Irina Dumitrescu}, \bibinfo{person}{Stefan Ropke}, \bibinfo{person}{Jean-FranÃ§ois Cordeau}, {and} \bibinfo{person}{Gilbert Laporte}.} \bibinfo{year}{2009}\natexlab{}.
\newblock \showarticletitle{The traveling salesman problem with pickup and delivery: Polyhedral results and a branch-and-cut algorithm}.
\newblock \bibinfo{journal}{\emph{Mathematical Programming}}  \bibinfo{volume}{121} (\bibinfo{date}{07} \bibinfo{year}{2009}), \bibinfo{pages}{269--305}.
\newblock
\urldef\tempurl%
\url{https://doi.org/10.1007/s10107-008-0234-9}
\showDOI{\tempurl}


\bibitem[Google(2023)]%
        {googleor}
\bibfield{author}{\bibinfo{person}{Google}.} \bibinfo{year}{2023}\natexlab{}.
\newblock \bibinfo{booktitle}{\emph{OR-tools, Google optimization tools}}.
\newblock
\urldef\tempurl%
\url{https://developers.google. com/optimization}
\showURL{%
\tempurl}


\bibitem[Gurobi(2023)]%
        {gurobi}
\bibfield{author}{\bibinfo{person}{Gurobi}.} \bibinfo{year}{2023}\natexlab{}.
\newblock \bibinfo{booktitle}{\emph{Mixed Integer Programming Basics}}.
\newblock
\urldef\tempurl%
\url{https://www.gurobi.com/resources/mixed-integer-programming-mip-a-primer-on-the-basics/}
\showURL{%
\tempurl}


\bibitem[Helsgaun(2017)]%
        {lkh3}
\bibfield{author}{\bibinfo{person}{Keld Helsgaun}.} \bibinfo{year}{2017}\natexlab{}.
\newblock \bibinfo{booktitle}{\emph{An Extension of the Lin-Kernighan-Helsgaun TSP Solver for Constrained Traveling Salesman and Vehicle Routing Problems.}}
\newblock


\bibitem[Joshi et~al\mbox{.}(2019)]%
        {Joshi2019graph}
\bibfield{author}{\bibinfo{person}{Chaitanya~K. Joshi}, \bibinfo{person}{Thomas Laurent}, {and} \bibinfo{person}{Xavier Bresson}.} \bibinfo{year}{2019}\natexlab{}.
\newblock \showarticletitle{An Efficient Graph Convolutional Network Technique for the Travelling Salesman Problem}.
\newblock \bibinfo{journal}{\emph{ArXiv}}  \bibinfo{volume}{abs/1906.01227} (\bibinfo{year}{2019}).
\newblock


\bibitem[Kool et~al\mbox{.}(2019)]%
        {kool2018attention}
\bibfield{author}{\bibinfo{person}{Wouter Kool}, \bibinfo{person}{Herke van Hoof}, {and} \bibinfo{person}{Max Welling}.} \bibinfo{year}{2019}\natexlab{}.
\newblock \showarticletitle{Attention, Learn to Solve Routing Problems!}. In \bibinfo{booktitle}{\emph{International Conference on Learning Representations}}.
\newblock


\bibitem[Li et~al\mbox{.}(2022)]%
        {lijing2022pd}
\bibfield{author}{\bibinfo{person}{Jingwen Li}, \bibinfo{person}{Liang Xin}, \bibinfo{person}{Zhiguang Cao}, \bibinfo{person}{Andrew Lim}, \bibinfo{person}{Wen Song}, {and} \bibinfo{person}{Jie Zhang}.} \bibinfo{year}{2022}\natexlab{}.
\newblock \showarticletitle{Heterogeneous Attentions for Solving Pickup and Delivery Problem via Deep Reinforcement Learning}.
\newblock \bibinfo{journal}{\emph{IEEE Transactions on Intelligent Transportation Systems}} \bibinfo{volume}{23}, \bibinfo{number}{3} (\bibinfo{year}{2022}), \bibinfo{pages}{2306--2315}.
\newblock
\urldef\tempurl%
\url{https://doi.org/10.1109/TITS.2021.3056120}
\showDOI{\tempurl}


\bibitem[Li et~al\mbox{.}(2021)]%
        {lixijun2021pd}
\bibfield{author}{\bibinfo{person}{Xijun Li}, \bibinfo{person}{Weilin Luo}, \bibinfo{person}{Mingxuan Yuan}, \bibinfo{person}{Jun Wang}, \bibinfo{person}{Jiawen Lu}, \bibinfo{person}{Jie Wang}, \bibinfo{person}{Jinhu Lü}, {and} \bibinfo{person}{Jia Zeng}.} \bibinfo{year}{2021}\natexlab{}.
\newblock \showarticletitle{Learning to Optimize Industry-Scale Dynamic Pickup and Delivery Problems}. In \bibinfo{booktitle}{\emph{2021 IEEE 37th International Conference on Data Engineering (ICDE)}}. \bibinfo{pages}{2511--2522}.
\newblock
\urldef\tempurl%
\url{https://doi.org/10.1109/ICDE51399.2021.00283}
\showDOI{\tempurl}


\bibitem[Lin and Kernighan(1973)]%
        {lin1973AnEH}
\bibfield{author}{\bibinfo{person}{Shen Lin} {and} \bibinfo{person}{Brian~W. Kernighan}.} \bibinfo{year}{1973}\natexlab{}.
\newblock \showarticletitle{An Effective Heuristic Algorithm for the Traveling-Salesman Problem}.
\newblock \bibinfo{journal}{\emph{Oper. Res.}}  \bibinfo{volume}{21} (\bibinfo{year}{1973}), \bibinfo{pages}{498--516}.
\newblock


\bibitem[Lu et~al\mbox{.}(2020)]%
        {lu2020A}
\bibfield{author}{\bibinfo{person}{Hao Lu}, \bibinfo{person}{Xingwen Zhang}, {and} \bibinfo{person}{Shuang Yang}.} \bibinfo{year}{2020}\natexlab{}.
\newblock \showarticletitle{A Learning-based Iterative Method for Solving Vehicle Routing Problems}. In \bibinfo{booktitle}{\emph{International Conference on Learning Representations}}.
\newblock


\bibitem[Ma et~al\mbox{.}(2021)]%
        {ma2021a}
\bibfield{author}{\bibinfo{person}{Yi Ma}, \bibinfo{person}{Xiaotian Hao}, \bibinfo{person}{Jianye HAO}, \bibinfo{person}{Jiawen Lu}, \bibinfo{person}{Xing Liu}, \bibinfo{person}{Xialiang Tong}, \bibinfo{person}{Mingxuan Yuan}, \bibinfo{person}{Zhigang Li}, \bibinfo{person}{Jie Tang}, {and} \bibinfo{person}{Zhaopeng Meng}.} \bibinfo{year}{2021}\natexlab{}.
\newblock \showarticletitle{A Hierarchical Reinforcement Learning Based Optimization Framework for Large-scale Dynamic Pickup and Delivery Problems}. In \bibinfo{booktitle}{\emph{Advances in Neural Information Processing Systems}}, \bibfield{editor}{\bibinfo{person}{A.~Beygelzimer}, \bibinfo{person}{Y.~Dauphin}, \bibinfo{person}{P.~Liang}, {and} \bibinfo{person}{J.~Wortman Vaughan}} (Eds.).
\newblock


\bibitem[Ma et~al\mbox{.}(2022)]%
        {ma2022pd}
\bibfield{author}{\bibinfo{person}{Yining Ma}, \bibinfo{person}{Jingwen Li}, \bibinfo{person}{Zhiguang Cao}, \bibinfo{person}{Wen Song}, \bibinfo{person}{Hongliang Guo}, \bibinfo{person}{Yuejiao Gong}, {and} \bibinfo{person}{Yeow~Meng Chee}.} \bibinfo{year}{2022}\natexlab{}.
\newblock \showarticletitle{Efficient Neural Neighborhood Search for Pickup and Delivery Problems}. In \bibinfo{booktitle}{\emph{Proceedings of the Thirty-First International Joint Conference on Artificial Intelligence, {IJCAI-22}}}. \bibinfo{pages}{4776--4784}.
\newblock


\bibitem[Miki et~al\mbox{.}(2018)]%
        {miki2018rl}
\bibfield{author}{\bibinfo{person}{Shoma Miki}, \bibinfo{person}{Daisuke Yamamoto}, {and} \bibinfo{person}{Hiroyuki Ebara}.} \bibinfo{year}{2018}\natexlab{}.
\newblock \showarticletitle{Applying Deep Learning and Reinforcement Learning to Traveling Salesman Problem}. In \bibinfo{booktitle}{\emph{2018 International Conference on Computing, Electronics \& Communications Engineering (iCCECE)}}. \bibinfo{pages}{65--70}.
\newblock


\bibitem[Nazari et~al\mbox{.}(2018)]%
        {nazari2018rl}
\bibfield{author}{\bibinfo{person}{Mohammadreza Nazari}, \bibinfo{person}{Afshin Oroojlooy}, \bibinfo{person}{Martin Tak\'{a}\v{c}}, {and} \bibinfo{person}{Lawrence~V. Snyder}.} \bibinfo{year}{2018}\natexlab{}.
\newblock \showarticletitle{Reinforcement Learning for Solving the Vehicle Routing Problem}. In \bibinfo{booktitle}{\emph{Proceedings of the 32nd International Conference on Neural Information Processing Systems}} \emph{(\bibinfo{series}{NIPS'18})}. \bibinfo{publisher}{Curran Associates Inc.}, \bibinfo{pages}{9861–9871}.
\newblock


\bibitem[Pacheco et~al\mbox{.}(2022)]%
        {pacheco2022pdtsp}
\bibfield{author}{\bibinfo{person}{Toni Pacheco}, \bibinfo{person}{Rafael Martinelli}, \bibinfo{person}{Anand Subramanian}, \bibinfo{person}{TÃºlio Toffolo}, {and} \bibinfo{person}{Thibaut Vidal}.} \bibinfo{year}{2022}\natexlab{}.
\newblock \showarticletitle{Exponential-Size Neighborhoods for the Pickup-and-Delivery Traveling Salesman Problem}.
\newblock \bibinfo{journal}{\emph{Transportation Science}}  \bibinfo{volume}{57} (\bibinfo{date}{10} \bibinfo{year}{2022}).
\newblock
\urldef\tempurl%
\url{https://doi.org/10.1287/trsc.2022.1176}
\showDOI{\tempurl}


\bibitem[Renaud et~al\mbox{.}(2002)]%
        {renaud2002pd}
\bibfield{author}{\bibinfo{person}{Jacques Renaud}, \bibinfo{person}{Fayez~F. Boctor}, {and} \bibinfo{person}{Gilbert Laporte}.} \bibinfo{year}{2002}\natexlab{}.
\newblock \showarticletitle{Perturbation heuristics for the pickup and delivery traveling salesman problem}.
\newblock \bibinfo{journal}{\emph{Computers \& Operations Research}} \bibinfo{volume}{29}, \bibinfo{number}{9} (\bibinfo{year}{2002}), \bibinfo{pages}{1129--1141}.
\newblock
\showISSN{0305-0548}
\urldef\tempurl%
\url{https://doi.org/10.1016/S0305-0548(00)00109-X}
\showDOI{\tempurl}


\bibitem[Renaud et~al\mbox{.}(2000)]%
        {renaud2000pd}
\bibfield{author}{\bibinfo{person}{Jacques Renaud}, \bibinfo{person}{Fayez~F. Boctor}, {and} \bibinfo{person}{Jamal Ouenniche}.} \bibinfo{year}{2000}\natexlab{}.
\newblock \showarticletitle{A heuristic for the pickup and delivery traveling salesman problem}.
\newblock \bibinfo{journal}{\emph{Computers \& Operations Research}} \bibinfo{volume}{27}, \bibinfo{number}{9} (\bibinfo{year}{2000}), \bibinfo{pages}{905--916}.
\newblock
\showISSN{0305-0548}
\urldef\tempurl%
\url{https://doi.org/10.1016/S0305-0548(99)00066-0}
\showDOI{\tempurl}


\bibitem[Ruland and Rodin(1997)]%
        {ruland1994pd}
\bibfield{author}{\bibinfo{person}{K.S. Ruland} {and} \bibinfo{person}{E.Y. Rodin}.} \bibinfo{year}{1997}\natexlab{}.
\newblock \showarticletitle{The pickup and delivery problem: Faces and branch-and-cut algorithm}.
\newblock \bibinfo{journal}{\emph{Computers \& Mathematics with Applications}} \bibinfo{volume}{33}, \bibinfo{number}{12} (\bibinfo{year}{1997}), \bibinfo{pages}{1--13}.
\newblock
\showISSN{0898-1221}
\urldef\tempurl%
\url{https://doi.org/10.1016/S0898-1221(97)00090-4}
\showDOI{\tempurl}


\bibitem[Savelsbergh(1990)]%
        {Savelsbergh1990pd}
\bibfield{author}{\bibinfo{person}{M.W.P. Savelsbergh}.} \bibinfo{year}{1990}\natexlab{}.
\newblock \showarticletitle{An efficient implementation of local search algorithms for constrained routing problems}.
\newblock \bibinfo{journal}{\emph{European Journal of Operational Research}} \bibinfo{volume}{47}, \bibinfo{number}{1} (\bibinfo{year}{1990}), \bibinfo{pages}{75--85}.
\newblock
\showISSN{0377-2217}


\bibitem[Skiena(1991)]%
        {skiena1991hc}
\bibfield{author}{\bibinfo{person}{Steven Skiena}.} \bibinfo{year}{1991}\natexlab{}.
\newblock \bibinfo{booktitle}{\emph{Implementing Discrete Mathematics: Combinatorics and Graph Theory with Mathematica}}.
\newblock \bibinfo{publisher}{Addison-Wesley Longman Publishing Co., Inc.}, \bibinfo{address}{USA}.
\newblock
\showISBNx{0201509431}


\bibitem[Veenstra et~al\mbox{.}(2017)]%
        {veenstra2017pdtsp}
\bibfield{author}{\bibinfo{person}{Marjolein Veenstra}, \bibinfo{person}{Kees~Jan Roodbergen}, \bibinfo{person}{Iris~F.A. Vis}, {and} \bibinfo{person}{Leandro~C. Coelho}.} \bibinfo{year}{2017}\natexlab{}.
\newblock \showarticletitle{The pickup and delivery traveling salesman problem with handling costs}.
\newblock \bibinfo{journal}{\emph{European Journal of Operational Research}} \bibinfo{volume}{257}, \bibinfo{number}{1} (\bibinfo{year}{2017}), \bibinfo{pages}{118--132}.
\newblock
\showISSN{0377-2217}
\urldef\tempurl%
\url{https://doi.org/10.1016/j.ejor.2016.07.009}
\showDOI{\tempurl}


\bibitem[Vinyals et~al\mbox{.}(2015)]%
        {vinyals2015ptr}
\bibfield{author}{\bibinfo{person}{Oriol Vinyals}, \bibinfo{person}{Meire Fortunato}, {and} \bibinfo{person}{Navdeep Jaitly}.} \bibinfo{year}{2015}\natexlab{}.
\newblock \showarticletitle{Pointer Networks}. In \bibinfo{booktitle}{\emph{Advances in Neural Information Processing Systems}}, \bibfield{editor}{\bibinfo{person}{C.~Cortes}, \bibinfo{person}{N.~Lawrence}, \bibinfo{person}{D.~Lee}, \bibinfo{person}{M.~Sugiyama}, {and} \bibinfo{person}{R.~Garnett}} (Eds.), Vol.~\bibinfo{volume}{28}. \bibinfo{publisher}{Curran Associates, Inc.}
\newblock


\bibitem[Wu et~al\mbox{.}(2021)]%
        {wu2019rl}
\bibfield{author}{\bibinfo{person}{Yaoxin Wu}, \bibinfo{person}{Wen Song}, \bibinfo{person}{Zhiguang Cao}, \bibinfo{person}{Jie Zhang}, {and} \bibinfo{person}{Andrew Lim}.} \bibinfo{year}{2021}\natexlab{}.
\newblock \showarticletitle{Learning Improvement Heuristics for Solving Routing Problems}.
\newblock \bibinfo{journal}{\emph{IEEE Transactions on Neural Networks and Learning Systems}}  \bibinfo{volume}{33} (\bibinfo{date}{03} \bibinfo{year}{2021}), \bibinfo{pages}{5057--5069}.
\newblock


\bibitem[Xin et~al\mbox{.}(2021)]%
        {xin2021neurolkh}
\bibfield{author}{\bibinfo{person}{Liang Xin}, \bibinfo{person}{Wen Song}, \bibinfo{person}{Zhiguang Cao}, {and} \bibinfo{person}{Jie Zhang}.} \bibinfo{year}{2021}\natexlab{}.
\newblock \showarticletitle{Neuro{LKH}: Combining Deep Learning Model with Lin-Kernighan-Helsgaun Heuristic for Solving the Traveling Salesman Problem}. In \bibinfo{booktitle}{\emph{Advances in Neural Information Processing Systems}}.
\newblock


\bibitem[Xing et~al\mbox{.}(2020)]%
        {Xing2020TS}
\bibfield{author}{\bibinfo{person}{Zhihao Xing}, \bibinfo{person}{Shikui Tu}, {and} \bibinfo{person}{Lei Xu}.} \bibinfo{year}{2020}\natexlab{}.
\newblock \showarticletitle{Solve Traveling Salesman Problem by Monte Carlo Tree Search and Deep Neural Network}.
\newblock \bibinfo{journal}{\emph{ArXiv}}  \bibinfo{volume}{abs/2005.06879} (\bibinfo{year}{2020}).
\newblock


\bibitem[Zheng et~al\mbox{.}(2020)]%
        {Zheng2020CombiningRL}
\bibfield{author}{\bibinfo{person}{Jiongzhi Zheng}, \bibinfo{person}{Kun He}, \bibinfo{person}{Jianrong Zhou}, \bibinfo{person}{Yan Jin}, {and} \bibinfo{person}{Chumin Li}.} \bibinfo{year}{2020}\natexlab{}.
\newblock \showarticletitle{Combining Reinforcement Learning with Lin-Kernighan-Helsgaun Algorithm for the Traveling Salesman Problem}.
\newblock \bibinfo{journal}{\emph{ArXiv}}  \bibinfo{volume}{abs/2012.04461} (\bibinfo{year}{2020}).
\newblock


\bibitem[Zheng et~al\mbox{.}(2023)]%
        {Zheng2023lkh}
\bibfield{author}{\bibinfo{person}{Jiongzhi Zheng}, \bibinfo{person}{Kun He}, \bibinfo{person}{Jianrong Zhou}, \bibinfo{person}{Yan Jin}, {and} \bibinfo{person}{Chu-Min Li}.} \bibinfo{year}{2023}\natexlab{}.
\newblock \showarticletitle{Reinforced Lin-Kernighan-Helsgaun algorithms for the traveling salesman problems}.
\newblock \bibinfo{journal}{\emph{Knowledge-Based Systems}}  \bibinfo{volume}{260} (\bibinfo{year}{2023}), \bibinfo{pages}{110144}.
\newblock
\showISSN{0950-7051}


\bibitem[Zong et~al\mbox{.}(2022)]%
        {zong2022multi}
\bibfield{author}{\bibinfo{person}{Zefang Zong}, \bibinfo{person}{Meng Zheng}, \bibinfo{person}{Yong Li}, {and} \bibinfo{person}{Depeng Jin}.} \bibinfo{year}{2022}\natexlab{}.
\newblock \showarticletitle{MAPDP: Cooperative Multi-Agent Reinforcement Learning to Solve Pickup and Delivery Problems}.
\newblock \bibinfo{journal}{\emph{Proceedings of the AAAI Conference on Artificial Intelligence}} \bibinfo{volume}{36}, \bibinfo{number}{9} (\bibinfo{year}{2022}), \bibinfo{pages}{9980--9988}.
\newblock
\urldef\tempurl%
\url{https://doi.org/10.1609/aaai.v36i9.21236}
\showDOI{\tempurl}


\end{thebibliography}
  
\section{Appendices}

\subsection{IP Formulation}
\label{append:ip}

In addition to the notation already introduced in the main text, we define $\delta(S)=\{(i,j)\in\mathcal{L}:i\in S, j \notin S \text{ or } i \notin S, j \in S \}$ for any set of nodes $S \subseteq \mathcal{N}$. If $S=\{i\}$, we write $\delta(i)$ instead of $\delta(\{i\})$. In addition, we distinguish between the starting depot and the ending depot by designating the ending depot as \(2n+1\). The PDTSP was formulated as a binary linear program \cite{dumi2009tsp} by associating a binary variable $x_{ij}$ with every edge $(i,j)\in\mathcal{L}$:
\begin{align}
\min_{x_{ij}: (i,j) \in \mathcal{L}} & \sum_{(i,j) \in \mathcal{L}} c_{ij} x_{ij} \\
s.t.\ & x_{0,2n+1}=1, \\ 
& x(\delta (i))=2, \forall i \in \mathcal{N} \label{eq:degree}\\
& x(\delta (S))\geq2, \forall S \in \mathcal{N}, 3 \leq |S| \leq |V|/2 \label{eq:sec_1} \\
& x(\delta (S) )\geq 4, \forall S\in \mathcal{U} \label{eq:precedence}\\
& x_{ij} \in \{0,1\}, \forall (i,j)\in \mathcal{L},
\end{align}
where, $\mathcal{U}$ is the collection of subsets $S \subset \mathcal{N}$ satisfying $3 \leq |S| \leq |\mathcal{N}| - 2$ with $0\in S, 2n+1\notin S$ and for which there exists $i\in \mathcal{P}$ such that $i\notin S$ and $n+i \in S$. Constraints (\ref{eq:degree}) are degree constraints, (\ref{eq:sec_1}) are subtour elimination constraints (SEC), and (\ref{eq:precedence}) are precedence constraints which ensure that node $i$ is visited before $n+i$ for every $i\in\mathcal{P}$.









\subsection{Proof for propositions}
\label{appendix:proof}
\noindent \textbf{Proof for Proposition 3.1}
\begin{proof}
    The number of permutations for all $\mathcal{P}$- and $\mathcal{D}$-nodes regardless of precedence constraints is $(2n)!$, which is the total number of Hamiltonian cycle in PDTSP. 

    Now we will count the number of tours. First, let us look at the number of permutations for $\mathcal{P}$-nodes. Since the size of $\mathcal{P}$-nodes is $n$, its permutation is $n!$. A tour consists of both $\mathcal{P}$- and $\mathcal{D}$-nodes. So now we need to insertion all the $n$ corresponding $\mathcal{D}$-nodes into each permutation, in order to count the total number of tours. 
    
    Let us focus on an arbitrary permutation of $\mathcal{P}$-nodes, denoted as $(i_1,\cdots,i_n), \forall i_1,\cdots, i_n\in{\mathcal{P}}$. 
    We look backwards from the last $\mathcal{P}$-node, i.e., $i_n$, in the permutation, and check the possible positions of its $\mathcal{D}$-node, i.e., node $n+i_n$. 
    To fulfil the precedence constraints, node $n+i_n$ can only be placed after $i_n$, which is $(i_1,\cdots,i_n,n+i_n)$. 
    Then we decide where to place node $n+i_{n-1}$, which is the $\mathcal{D}$-node of the second last $\mathcal{P}$-node, i.e., $i_{n-1}$ in the permutation. 
    Constrained by the precedence constraints, node $n+i_{n-1}$ can be placed in three slots, namely, 
    (1) after node $n+i_n$, i.e., $(i_1,\cdots,i_n,n+i_n,n+i_{n-1})$, 
    (2) before node $n+i_n$, i.e., $(i_1,\cdots,i_n,n+i_{n-1},n+i_n)$, 
    and (3) before node $i_n$, i.e., $(i_1,\cdots,n+i_{n-1},i_n,n+i_n)$. 
    This process continues when we finish checking every $\mathcal{D}$-node. 
    By induction, the number of all tours for each $\mathcal{P}$-node permutation $(i_1,\cdots,i_n)$ is $1\cdot 3 \cdot \cdots\cdot (2n-1)=\frac{(2n)!}{2^n\cdot n!}$. Therefore, the total number of tours is $\frac{(2n)!}{2^n}$.
\end{proof}

\noindent \textbf{Proof for Proposition 3.2}
\begin{proof}
According to Definition \ref{defn:tour} and \ref{defn:block}, a tour $\mathcal{T}$ can be recast as $0 \rightarrow \mathcal{T}^{(\mathcal{P})}_{1} \rightarrow \cdots \rightarrow\mathcal{T}^{(\mathcal{P})}_{k} \rightarrow \cdots \rightarrow \mathcal{T}^{(\mathcal{D})}_{K} \rightarrow 0$ where $k$ denotes the index of a block in the block sequence, and $K$ is the total number of blocks.    
\end{proof}

\noindent \textbf{Proof for Proposition 3.3}
\begin{proof}
    According to Definition \ref{defn:block}, $\forall \mathcal{T}^{(\mathcal{P})}: i_1 \rightarrow \cdots \rightarrow i_{u} \rightarrow i_{v} \rightarrow \cdots \rightarrow i_l (l\geqslant2)$ , it is decomposed as two adjacent $\mathcal{P}$-blocks $i_1 \rightarrow \cdots \rightarrow i_{u}$ and $i_v \rightarrow \cdots \rightarrow i_{l}$. Vice versa, two adjacent $\mathcal{P}$-blocks can be merged into one. 
\end{proof}

\noindent \textbf{Proof for Proposition 3.4}
\begin{proof}
    We consider $\mathcal{T}^{(\mathcal{P})}_{k_1}: i_1\rightarrow \cdots \rightarrow i \rightarrow \cdots \rightarrow i_l$ and $\mathcal{T}^{(\mathcal{D})}_{k_2}: n+j_1 \rightarrow \cdots \rightarrow n+i \rightarrow \cdots \rightarrow n+j_m$. If $k_2<k_1$, we have $d_{n+i} \leqslant d_{n+j_m}<p_{i_1}\leqslant p_{i}$, which violates precedence constraints in Definition \ref{defn:tour}. Therefore, we have $k_1<k_2$.
\end{proof}
\noindent \textbf{Proof for Proposition 4.2}
\begin{proof}
    We only need to show that the operator does not change the feasibility of the Hamiltonian cycle. We look into $\mathcal{P}$-block $\mathcal{T}^{(\mathcal{P})}_{k}$ in a tour $\mathcal{T}$ where $\mathcal{T}^{(\mathcal{P})}_{k}: i_1\rightarrow \cdots \rightarrow i_l$. We have $\forall i_u, i_v \in \mathcal{N}(\mathcal{T}^{(\mathcal{P})}_{k})$,
    \begin{align}
        & p_{i_u}\leqslant p_{i_l}, p_{i_v} \leqslant p_{i_l} \nonumber 
    \end{align}
    According to Proposition \ref{prop:tour_precedence}, We have:
    \begin{align}
        \forall i_u, i_v \in \mathcal{N}(\mathcal{T}^{(\mathcal{P})}_{k}), n+i_u, n+i_v \in \cup_{k'>k} \mathcal{T}^{(\mathcal{D})}_{k'} \nonumber
    \end{align}
    Therefore, $d_{n+i_u}, d_{n+i_v} \geqslant p_{i_l}+1$. We then have
    \begin{align}
        & p'_{i_u}=p_{i_v}< p_{i_l}+1 \leqslant d_{n+i_u} \nonumber  \\
        & p'_{i_v}=p_{i_u}< p_{i_l}+1 \leqslant d_{n+i_v}. \nonumber 
    \end{align}
    According to Definition \ref{defn:tour}, the new Hamiltonian cycle is still feasible. We omit the proof for the $\mathcal{D}$-block.
\end{proof}
\noindent \textbf{Proof for Proposition 4.3}
\begin{proof} 
     We look into the $\mathcal{P}$-block $\mathcal{T}^{(\mathcal{P})}_k:i_1 \rightarrow \cdots \rightarrow i_l$ and the $\mathcal{D}$-block $\mathcal{T}^{(\mathcal{P})}_{k'}:n+j_1 \rightarrow \cdots \rightarrow n+j_m, \forall k>k'$ in the tour $T$. We have $\forall i_u \in \mathcal{N}(\mathcal{T}^{(\mathcal{P})}_k),n+j_r \in \mathcal{N}(\mathcal{T}^{(\mathcal{D})}_{k'})$,
    \begin{align}
        & p'_{i_u}=d_{n+j_r}<p_{i_1}\leqslant p_{i_u}<d_{n+i_u} \nonumber \\
        & d'_{n+j_r}=p_{i_u}>d_{n+j_1} \geqslant d_{n+j_r}>p_{j_r}\nonumber 
    \end{align}
    According to Definition \ref{defn:tour}, the new sequence is feasible. 
\end{proof}

\noindent \textbf{Proof for Proposition 4.4}
\begin{proof}
  We have 
    \begin{align}
        p'_i=p_j<d_{n+j} =d'_{n+i}, \ 
        p'_j=p_i<d_{n+i} =d'_{n+j}. \nonumber
    \end{align}
    Therefore, the new Hamiltonian cycle is feasible. 
\end{proof}
\noindent \textbf{Proof for Proposition 4.5}
\begin{proof}
    We denote the indices of $i$ and $n+i$ in the original tour $\mathcal{T}$ as $p_i$ and $d_{n+i}(p_i<d_{n+i})$, respectively. We only need to show that $i$ and $n+i$ can be placed on any positions $p'_i$ and $d'_{n+i}(p'_i<d'_{n+i})$ in the new tour $\mathcal{T}'$ after the execution of node exchange operators. We look into the $\mathcal{P}$-node $i$. We first consider case 1 when $p'_i<p_i$. If the node at position $p'_i$ in tour $\mathcal{T}$ is a $\mathcal{D}$-node, we execute the InterBlock N$_X$O (N2) to swap the $\mathcal{D}$-node and $i$. If the node at position $p'_i$ in tour $\mathcal{T}$ is a $\mathcal{P}$-node, we execute N1 or N3 to swap the $\mathcal{P}$-node and $i$. We then consider case 2 when $p'_i>p_i$. If the node at position $p'_i$ in tour $\mathcal{T}$ is a $\mathcal{P}$-node, we execute N1 or N3 to swap the $\mathcal{P}$-node and $i$. If the node at position $p'_i$ in tour $\mathcal{T}$ is a $\mathcal{D}$-node, we first execute the N$^2_X$O (N3) to swap node $i$ and any $\mathcal{P}$-node after the $\mathcal{D}$-node. The index of $i$ becomes $p_i^{I}$ where $p'_i<p_i^{I}$. We execute the same operators in case 1 and the index of $i$ then becomes $p'_i$. We omit the details of proof regarding any arbitrary $\mathcal{D}$-node $n+i$. 
\end{proof}
\begin{figure*}
    \centering
    \subfloat[$\mathcal{N}=11$]{
   \includegraphics[scale=.32]{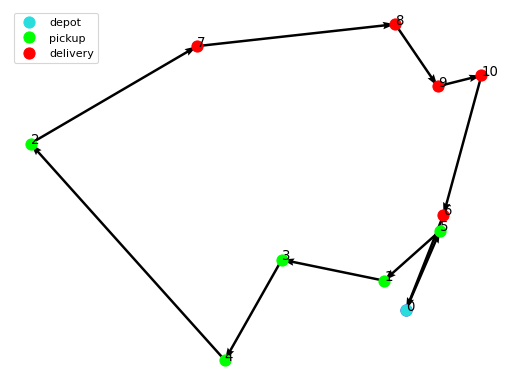}}
   \subfloat[$\mathcal{N}=21$]{
   \includegraphics[scale=.32]{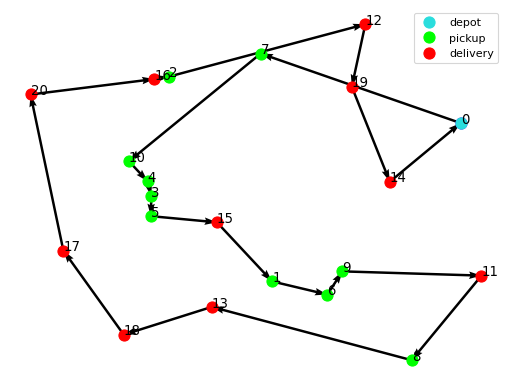}}
   \subfloat[$\mathcal{N}=41$]{\includegraphics[scale=.32]{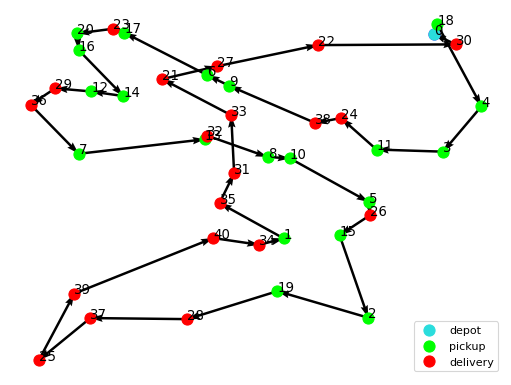}}
   
   \subfloat[$\mathcal{N}=61$]{
   \includegraphics[scale=.32]{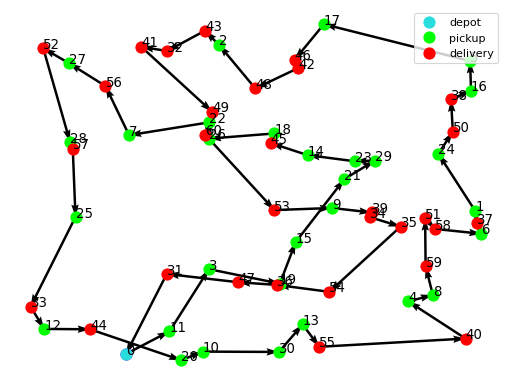}}
   \subfloat[$\mathcal{N}=101$]{\includegraphics[scale=.32]{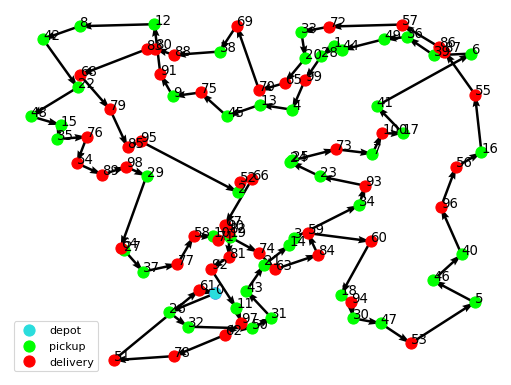}}
    \caption{Minimum cost tour of instances with $|\mathcal{N}|=11,21,41,61,101$}
    \label{fig:opt_tour_app}
\end{figure*}

\noindent \textbf{Proof for Proposition 4.6}
\begin{proof}
    We look into a block sequence $\mathcal{T}^{(\mathcal{P})}_{k_1} \rightarrow \cdots \rightarrow \mathcal{T}^{(\mathcal{P})}_{k_h}$ in a tour. $\mathcal{T}^{(\mathcal{P})}_{k_u}$ and $\mathcal{T}^{(\mathcal{P})}_{k_v} (1 \leqslant u, v \leqslant h)$ are swapped. $\forall i_u \in \mathcal{N}(\mathcal{T}^{(\mathcal{P})}_{k_u}), i_v \in \mathcal{N}(\mathcal{T}^{(\mathcal{P})}_{k_v})$,
    \begin{align}
        & p'_{i_u},p'_{i_v},p_{i_u},p_{i_v}\leqslant \max \{p_{i_h}:i_h \in \mathcal{N}(\mathcal{T}^{(\mathcal{P})}_{k_h}) \}. \nonumber 
    \end{align}
    According to Proposition \ref{prop:tour_precedence}, We have:
    \begin{align}
         n+i_u, n+i_v \in \cup_{k'>k_h} \mathcal{T}^{(\mathcal{D})}_{k'} \nonumber
    \end{align}
    Therefore, $d_{n+i_u}, d_{n+i_v} \geqslant \max \{p_{i_h}:i_h \in \mathcal{N}(\mathcal{T}^{(\mathcal{P})}_{k_h}) \}+1$. We then have
    \begin{align}
        p'_{i_u}< d_{n+i_u}, \ p'_{i_v}< d_{n+i_v}. \nonumber 
    \end{align}
    According to Definition \ref{defn:tour}, the new Hamiltonian cycle is feasible. We omit the proof for the $\mathcal{D}$-block sequence.
\end{proof}
\noindent \textbf{Proof for Proposition 4.7}
\begin{proof} 
     We look into the $\mathcal{P}$-block $\mathcal{T}^{(\mathcal{P})}_{k_r}:i_1 \rightarrow \cdots \rightarrow i_l$ and the $\mathcal{D}$-block $\mathcal{T}^{(\mathcal{D})}_{k_s}:n+j_1 \rightarrow \cdots \rightarrow n+j_m, k_r>k_s$ in the tour $T$. We have $\forall i_u \in \mathcal{N}(\mathcal{T}^{(\mathcal{P})}_{k_r}),n+j_r \in \mathcal{N}(\mathcal{T}^{(\mathcal{D})}_{k_s})$,
    \begin{align}
        & p'_{i_u}< p_{i_l}+1 < d_{n+i_u} \nonumber \\
        & d'_{n+j_r}>d_{n+j_1}-1>p_{j_r}\nonumber 
    \end{align}
    Therefore, the new Hamiltonian cycle is feasible.
\end{proof}



\end{document}